\documentclass{article}

\usepackage{microtype}
\usepackage{graphicx}
\usepackage{subfigure}
\usepackage{booktabs} 

\usepackage{hyperref}

\usepackage[accepted]{icml2024}

\usepackage{amsmath,bbm,bm}
\usepackage{amssymb}
\usepackage{mathtools}
\usepackage{amsthm}
\usepackage{adjustbox}
\usepackage{circledsteps}
\usepackage{stfloats}
\usepackage{soul}
\usepackage{multirow}
\usepackage{cuted}
\usepackage{tikz}
\usetikzlibrary{arrows.meta, positioning, shapes.geometric}

\usepackage[capitalize,noabbrev]{cleveref}

\theoremstyle{plain}
\newtheorem{theorem}{Theorem}[section]
\newtheorem{proposition}[theorem]{Proposition}
\newtheorem{lemma}[theorem]{Lemma}
\newtheorem{corollary}[theorem]{Corollary}
\theoremstyle{definition}

\theoremstyle{remark}

\def\argmin{\mathop{\rm argmin}}

\newcommand*\circled[1]{\tikz[baseline=(char.base)]{
            \node[shape=circle,draw,inner sep=1pt] (char) {#1};}}

\definecolor{mycol}{RGB}{0, 0, 0}

\usepackage[textsize=tiny]{todonotes}
\graphicspath{{fig/}}

\icmltitlerunning{Efficient Online Set-valued Classification with Bandit Feedback}

\begin{document}

\twocolumn[
\icmltitle{Efficient Online Set-valued Classification with Bandit Feedback}



\icmlsetsymbol{equal}{*}

\begin{icmlauthorlist}
\icmlauthor{Zhou Wang}{sch}
\icmlauthor{Xingye Qiao}{sch}
\end{icmlauthorlist}

\icmlaffiliation{sch}{Department of Mathematics and Statistics, Binghamton University, New York, USA}

\icmlcorrespondingauthor{Xingye Qiao}{xqiao@binghamton.edu}

\icmlkeywords{Machine Learning, ICML}

\vskip 0.3in
]



\printAffiliationsAndNotice{ } 

\begin{abstract}
Conformal prediction is a distribution-free method that wraps a given machine learning model and returns a set of plausible labels that contain the true label with a prescribed coverage rate. In practice, the empirical coverage achieved highly relies on fully observed label information from data both in the training phase for model fitting and the calibration phase for quantile estimation. This dependency poses a challenge in the context of online learning with bandit feedback, where a learner only has access to the correctness of actions (i.e., pulled an arm) but not the full information of the true label. In particular, when the pulled arm is incorrect, the learner only knows that the pulled one is not the true class label, but does not know which label is true. Additionally, bandit feedback further results in a smaller labeled dataset for calibration, limited to instances with correct actions, thereby affecting the accuracy of quantile estimation. To address these limitations, we propose Bandit Class-specific Conformal Prediction (BCCP), offering coverage guarantees on a class-specific granularity. Using an unbiased estimation of an estimand involving the true label, BCCP trains the model and makes set-valued inferences through stochastic gradient descent. Our approach overcomes the challenges of sparsely labeled data in each iteration and generalizes the reliability and applicability of conformal prediction to online decision-making environments.
 
\end{abstract}

\section{Introduction}\label{sec:intro}
Machine learning models, while highly effective, can fail in complicated scenarios due to inherent uncertainties and hence lead to irreversible consequences, particularly in high-stake applications. For instance, in autonomous vehicle systems, misidentifying real obstacles as harmless shadows on the road potentially causes abrupt braking or even dangerous maneuvers. In medical diagnostics, the challenge of differentiating between benign and malignant tumors in ambiguous cases can result in critical misdiagnoses, influencing treatment decisions. Such scenarios underscore the need for models capable of cautiously handling those observations with high uncertainty.

Quantifying the uncertainty associated with each observation can be addressed by reporting a prediction set, which can be realized by some set-valued classification paradigms such as Classification with the Reject Option \cite{herbei2006classification,bartlett2008classification,charoenphakdee2021classification,zhang2018reject} and Conformal Prediction \cite{vovk2005algorithmic,shafer2008tutorial,balasubramanian2014conformal}. Intuitively speaking, an observation with a large prediction set indicates its intrinsic difficulty and it is hard to be correctly classified. Unlike Classification with the Reject Option, the Conformal Prediction method particularly yields a set with valid prediction coverage, i.e., a prediction set includes the true label with a user-prescribed coverage rate $1-\alpha, \alpha\in[0, 1]$.

The literature on Conformal Prediction (and other set-valued classification methods) covers various aspects. For instance, \citet{lei2013distribution,lei2014classification,lei2015conformal,lei2018distribution,sadinle2019least,wang2018learning,wang2022set} consider the coverage guarantees conditional on each class instead of the standard marginal coverage \citep{vovk2005algorithmic}. \citet{romano2020classification,angelopoulos2021uncertainty} explore different (un)conformity scores to output informative conformal prediction sets.  \citet{tibshirani2019conformal} introduces weighted conformal prediction in the situation of covariate distribution shift, while \citet{hechtlinger2018cautious,guan2022prediction,JMLR:v24:23-0712} generalize set-valued predictions to the realm of out-of-distribution detection due to the semantic distribution shift by admitting an empty prediction set. However, these studies predominantly focus on the setting with access to full-label information and offline training, limiting their applicability in real-world scenarios.

Recent extensions of Conformal Prediction to online learning settings, (1) address arbitrary distribution shifts \citep{gibbs2021adaptive,gibbs2022conformal,zaffran2022adaptive,bhatnagar2023improved}, and (2) apply the principles on the off-policy evaluation problem \citep{taufiq2022conformal,zhang2023conformal,stanton2023bayesian} in reinforcement learning. Yet, these works require significantly more label information than what bandit feedback affords: in the distribution shift problem with full feedback, a learner knows the true label regardless of its decision's correctness; in the policy evaluation problem, the learner receives a reward that reflects the optimality of the pulled arm. In contrast, in the bandit feedback setting \citep{langford2007epoch,kakade2008efficient,wang2010potential} a learner only receives feedback about the correctness of predictions rather than the ground truth of label information. For instance, a learner in TikTok can correctly capture a positive attitude toward the video recommendation through a user's click, whereas the user's preferences remain uncertain if the presented recommendation is disliked by the user (it does not know what the user likes). Similarly, in personalized medicine, a medical system adjusts chemotherapy treatments based on partial feedback, such as tumor response, without full knowledge of how other treatments might have worked for that patient.

Motivated by the limited literature on Conformal Prediction within the context of online bandit feedback, we introduce the Bandit Class-specific Conformal Prediction (BCCP) framework for the multi-class classification problem. To the best of our knowledge, this is the first effort in applying conformal prediction to this particular context. Our key contributions are as follows: (1) BCCP leverages an unbiased estimator for accurate ground truth inference of label information, allowing the use of those data instances for which the wrong arm was pulled in both model fitting and quantile estimation; (2) Our method capitalizes on the efficiency of stochastic gradient descent for dynamically updating the quantile estimation, which differentiates itself from the traditional split conformal method in which sample quantiles based on a sufficiently large calibration dataset are used; (3) We theoretically prove that both the class-specific coverage and the excess risk with respect to the check loss converge at a rate of $\mathcal{O}(T^{-1/2})$ under certain conditions; (4) Recognizing the practical challenge of selecting an optimal learning rate for updating the quantile estimation, we use an ensemble approach to update the estimation with a range of learning rates; (5) The effectiveness of BCCP is empirically validated using three different score functions and two policies (for pulling arm) across three datasets, demonstrating the versatility and efficacy of our proposed framework.

The rest of the paper is organized as follows. In \cref{sec:review}, we begin with a review of the related work. This is followed by \cref{sec:method}, where we introduce our methodology complemented by a series of associated theorems. In \cref{sec:expir}, we present experiments to demonstrate the effectiveness of our method. The conclusions to our work are given in  \cref{sec:conclusion} and proofs are attached in \cref{app:implent}.

\section{Preliminary}\label{sec:review}
In this section, we review some key concepts of Conformal Prediction and the Multi-armed Bandit Problem.

\subsection{Conformal Prediction}\label{sec:CP}
Conformal prediction \citep{vovk2005algorithmic, lei2015conformal} is a distribution-free methodology that can complement various machine learning models, such as neural networks, support vector machines (SVMs), and random forests. It is utilized to produce set-valued predictions with a theoretically guaranteed coverage rate prescribed by users.

Consider a labeled training dataset $\mathcal{D}=\{(\bm X_{i}, Y_{i})\}_{i\in \mathcal{I}}$ ($\mathcal{I}$ denotes the index set) and a test instance $\bm X$ with unknown label $Y$, where both are assumed to be i.i.d. from an unknown distribution over the domain $\mathcal{X}\times \mathcal{Y}$. In the classification problem, the Standard Conformal Prediction employs a mapping (depending on the dataset $\mathcal{D}$) $\widehat{\mathcal{C}}:\mathcal{X}\mapsto 2^{\mathcal{Y}}$ and returns a prediction set $\widehat{\mathcal{C}}(\bm X)$ for the test point $\bm X$, ensuring the marginal coverage rate
\begin{equation}\label{eq:stdCP}
\mathbb{P}(Y\in \widehat{\mathcal{C}}(\bm X))\geq 1-\alpha,
\end{equation}
where $\alpha\in[0, 1]$ represents the pre-specified nominal non-coverage rate by practitioners. Notice that the probability is taken over the training dataset $\mathcal{D}$ and the test point $(\bm X, Y)$.

Considering that the marginal coverage guarantee in Standard Conformal Prediction may not be adequate for certain specific classes, \citet{lei2013distribution,lei2015conformal,lei2018distribution,sadinle2019least} explored Class-conditional Conformal Prediction, which offers class-specific coverage
\begin{equation}\label{eq:condCP}
\mathbb{P}(Y\in \widehat{\mathcal{C}}(\bm X)\mid Y=k)\geq 1-\alpha, ~\forall~k\in\mathcal{Y}.
\end{equation}
The same paradigm is also considered in \citet{wang2018learning,wang2022set,JMLR:v24:23-0712}. It is crucial to understand that while \eqref{eq:condCP} implies \eqref{eq:stdCP}, the converse is not necessarily true. On the other hand, compared to the marginal coverage, the class-specific coverage may yield larger prediction sets when practitioners have limited data for each class. Motivated by this limitation, \citet{ding2024class} proposed Clustered Conformal Prediction to navigate this trade-off between marginal and class-specific coverage in the low-data regime, while \citet{romano2020classification,angelopoulos2021uncertainty} proposed different score functions to improve the prediction set size especially when there are many classes.

In general, Conformal Prediction starts with a (conformity) score function $s:\mathcal{X}\times\mathcal{Y}\mapsto\mathbb{R}$. It is employed to gauge the proximity of an observation $\bm X$ to any class $k\in\mathcal{Y}$. Intuitively speaking, the larger the conformity score $s(\bm X, k)$, the higher the likelihood that the observation $\bm X$ belongs to the class $k$. This score function can manifest in various forms, such as the softmax probability in neural networks, the functional margin in SVMs, or the average predicted class probabilities of trees in random forests.

In the split conformal method \citep{papadopoulos2002inductive,lei2013distribution}, the index set $\mathcal{I}$ associated with the original dataset $\mathcal{D}$ is partitioned into two disjoint subsets: the training part $\mathcal{I}_{tr}$ and the calibration part $\mathcal{I}_{cal}$. The former is used to fit a model $\bm f$ in the training phase, such as training a neural network to minimize cross-entropy loss \citep{romano2020classification,angelopoulos2021uncertainty}, training an SVM to minimize hinge loss \citep{wang2018learning,wang2022set}, or growing a random forest based on Gini-impurity \citep{guan2022prediction}. This model $\bm f$ is then utilized to customize the aforementioned conformity score function $s$.  For example, $s$ could be directly taken as $\bm f$, or a monotonic function of $\bm f$, e.g., softmax score. With the conformity score established, the next step involves identifying score thresholds $\tau_{k},k\in\mathcal{Y}$ within the calibration part $\mathcal{I}_{cal}$, thereby enabling decision-making for the upcoming test points. In summary, a prediction set for a query $\bm X$ with the class-specific coverage guarantee \eqref{eq:condCP} is defined as
\begin{equation*}
\widehat{\mathcal{C}}(\bm X):=\{k\in\mathcal{Y}: s(\bm X, k)\geq \tau_{k}\},
\end{equation*}
where the threshold $\tau_{k}$ is determined as the $100\times\alpha\%$ sample quantile of the conformity scores for the calibration set, i.e., the $(\lfloor|\mathcal{I}_{cal}|\alpha\rfloor+1)$-th smallest value in $\{s(\bm X_i, k)\}_{i\in\mathcal{I}_{cal}}$ \citep{romano2019conformalized}. Throughout this article,  $|\cdot|$ being applied on a set denotes the size or cardinality of the set.

\subsection{Multi-armed Bandit and Multi-class Classification}
The Multi-armed Bandit Problem \citep{lai1985asymptotically,auer2002finite} is a fundamental concept in reinforcement learning. It presents a scenario where a learner aims to optimize rewards or minimize regrets (cumulatively assessed from feedback) by pulling an ``arm" (or taking an action), $A$, from a set of available arms denoted as $\{1, \cdots, K\}$, where $K$ represents the total number of arms. The selection of an arm is guided by a policy $\pi$, tailored to maximize expected gains over time. The policy $\pi$ could be a probability distribution to generate an arm to pull, or deterministic.

When extended to multi-class classification with bandit feedback, this concept incorporates contextual information or features, $\bm X$, effectively transforming it into a contextual bandit problem. Particularly in online learning settings, at time point $t$, the learner selects an arm $A_t\sim\pi$ for a given query context $\bm X_t$, and subsequently receives binary feedback $\mathbbm{1}\{A_t=Y_t\}$. This feedback, indicating whether the pulled arm (class) matches the true label $Y_t$, introduces uncertainty regarding the true label, complicating the learner's updating process. For example, different from the full feedback setting \citep{gibbs2021adaptive,gibbs2022conformal,bhatnagar2023improved}, the learner here has no idea upon the true label for the query $\bm X_t$ if the value of feedback is 0.

Several studies have explored the domain of contextual bandits, where the hypothesis space comprises linear predictors \citep{kakade2008efficient,wang2010potential,crammer2013multiclass,abbasi2011improved,gollapudi2021contextual,van2021beyond}. These works focus on the efficacy of linear models in capturing the relationship between context and action rewards. However, the linear representation has its limitations in capturing complex relationships.

In response to these limitations, recent studies have delved into neural contextual bandits \citep{zhou2020neural,jin2021mots,zhang2021neural,xu2022neural}. These approaches leverage the expressive power of deep neural networks to model the context-action relationship more effectively. There are various policies proposed, including Thompson sampling and Upper Confidence Bound algorithms, to navigate the bandit problem in more complex and non-linear environments.

Despite these advancements in reinforcement learning, the existing literature primarily focuses on point prediction and lacks mechanisms for set-valued prediction and coverage control. This gap is particularly concerning in critical domains, as discussed in \cref{sec:intro}. The issue is partially addressed by recent works \citep{taufiq2022conformal,zhang2023conformal,stanton2023bayesian}, which apply Conformal Prediction to off-policy evaluation problems, thereby returning prediction sets. However, these researches diverge from our work, which specifically addresses the bandit problem setting. Our focus lies in integrating set-valued predictions with the bandit feedback framework, an area that has not been extensively explored, presenting both novel challenges and opportunities for advancing the field.

\begin{color}{mycol}
\subsection{Set-valued Classification with Bandit Feedback}
The proposed BCCP method (summarized in \cref{alg}) aims to make set-valued decisions with a coverage guarantee for instances from the same distribution as the training data in the bandit feedback setting. Particularly, given a query $\bm X_t$, the learner pulls an arm $A_t$ and receives the feedback $\mathbbm{1}\{A_t=Y_t\}$. With this feedback, the learner updates the model and thresholds in conformal prediction (lines 4--5 in \cref{alg}). During the test phase, the learner returns the prediction set based on the trained model and thresholds (line 3 in \cref{alg}). 

Take healthcare as an example. Due to cost and safety concerns, insurance companies may only allow the healthcare provider to prescribe one diagnostic test (e.g., X-ray, followed by CT, followed by cancer biomarker blood test, etc.) at a time to the patient (this may be viewed as pulling a single arm). When a diagnostic test turns out negative for a suspect cause, it is still unknown what the cause really is (this is consistent with our setting in which the learner only receives a bandit feedback that confirms the correctness of the pulled arm but does not necessarily reveal the true label). After a series of training over a large number of patients has been conducted, we have a diagnostic system that can make predictions for a new patient based on the patient's profile. Unless for clear-cut cases, often it is much safer for the provider to consider a set of most plausible causes and design the treatment plan that considers all plausible diseases, as opposed to treating the patient based on one single predicted disease.
\end{color}

\section{Towards the Bandit Conformal}\label{sec:method}
In this section, we introduce our method: Bandit Class-specific Conformal Prediction, specifically designed for set-valued multi-class classification problems in an online bandit feedback setting: let $\{(\bm X_t, Y_t)\}_{t=1}^T$ be a sequence of i.i.d. \textcolor{mycol}{points} from the domain $\mathcal{X}\times\mathcal{Y}$, where \textcolor{mycol}{a leaner cannot observe the label $Y_t$ and receives the non-zero feedback} only when an arm is correctly pulled. We aim to report a prediction set $\widehat{\mathcal{C}}^{t-1}(\bm X_t)$ (the learner only uses the information up to time $t-1$) with a class-specific coverage guarantee.  

Our methodology entails three pivotal steps: (1) estimating a ground truth based on a policy and feedback, (2) training the model with this estimation, and (3) estimating the $100\times\alpha\%$ quantile $\tau_{k}$ for each class $k \in \mathcal{Y}$.

\subsection{Estimating $\mathbbm{1}\{Y_t=k\}$}
In the bandit feedback context, for each query instance $\bm X_t$, the learner pulls an arm $A_t \in \mathcal{Y}$ based on a given policy $\pi_t:= \pi_t(\cdot \mid \bm X_t)$, effectively making an educated guess about the potential true label. The environment then provides binary feedback indicating the correctness of the chosen arm, i.e., $\mathbbm{1}\{A_t=Y_t\}$. As a direct observation of  $Y_t$ is not available, we rely on the following estimation to $\mathbbm{1}\{Y_t=k\}$, i.e., 
\begin{equation*}
\Delta_{t,k}:=\frac{\mathbbm{1}\{A_t=k\}}{\pi_t(k\mid \bm X_t)}\mathbbm{1}\{A_t=Y_t\}.
\end{equation*}
\begin{proposition}\label{thm:indicator}
$\Delta_{t,k}$ serves as an unbiased estimator of $\mathbbm{1}\{Y_t=k\}$. This is substantiated by the equation $$\begin{aligned}
\mathbb{E}_{\pi_t}\bigl[\Delta_{t,k}\bigr]&=\mathbb{E}_{\pi_t}\bigl[\Delta_{t,k}\mid A_t=k\bigr]\cdot \pi_t(k\mid \bm X_t)\\
&~~~~+\mathbb{E}_{\pi_t}\bigl[\Delta_{t,k}\mid A_t\neq k\bigr]\cdot[1-\pi_t(k\mid \bm X_t)]\\
&=\frac{\mathbbm{1}\{k=Y_t\}}{\pi_t(k\mid \bm X_t)}\cdot\pi_t(k\mid \bm X_t)+0=\mathbbm{1}\{Y_t=k\},
\end{aligned}$$ where the expectation is taken with respect to policy $\pi_t$, conditioning on all previous information and the point $(\bm X_t, Y_t)$. 
\end{proposition}

This estimation framework lays the groundwork for subsequent tasks in our study. It allows us to effectively utilize the policy's capability to learn the real data-generating process without explicit knowledge about the true label $Y_t$.  

Policy design can be a flexible process, influenced by specific preferences such as the pursuit of simplicity or the goal of minimizing estimation variance.  In our research, we theoretically analyze the performance of certain policies characterized by the associated properties, as detailed in \cref{thm:finercvg,thm:finechk_reg}. Additionally, we conduct empirical evaluations and compare the performances of two distinct policies: the softmax policy (softmax probability output from a neural network as defined in \eqref{eq:softmax}) and the uniform policy (uniform distribution). See \cref{sec:expir}.

\begin{figure*}[!th]
    \centering
     \includegraphics[width=1\linewidth]{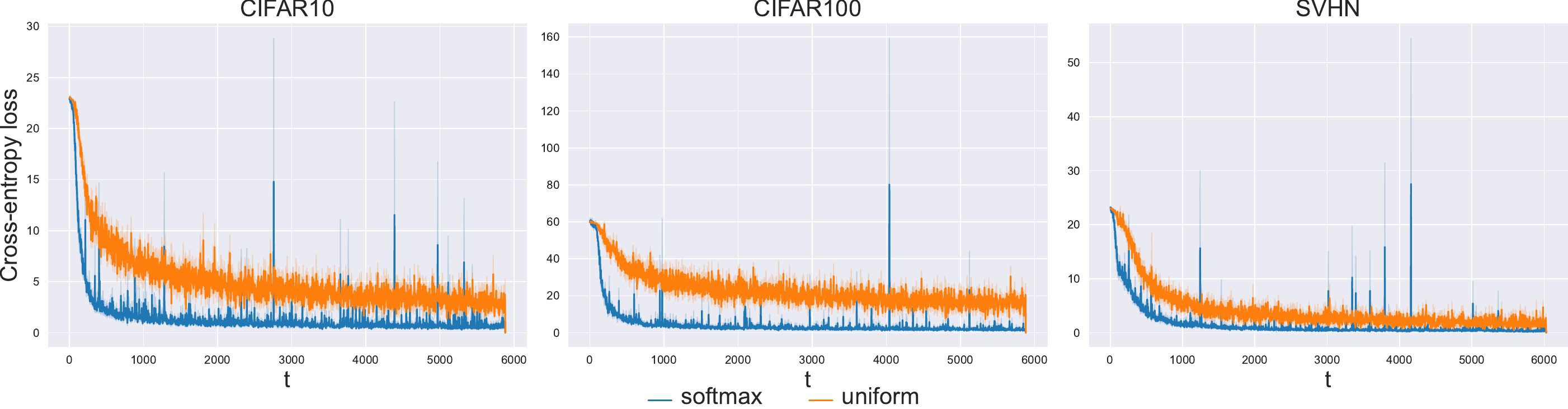}\vspace{-10pt}
    \caption{Accumulative cross-entropy loss under softmax policy and uniform policy.}\label{fig:CE}
\end{figure*}

\subsection{The Cross-entropy Loss with Bandit Feedback}\label{sec:banditCE}

Throughout this article, we train a neural network model $\bm f_{\mathcal{W}}(\bm X) = (f^1_{\mathcal{W}}(\bm X), \cdots, f^{|\mathcal{Y}|}_{\mathcal{W}}(\bm X))^\top \in \mathbb{R}^{|\mathcal{Y}|}$, which is parameterized by a set of matrices collectively represented by $\mathcal{W}$. Our primary objective in the training phase, particularly within the bandit feedback context, is to minimize a modified version of the cross-entropy loss for each input query $\bm X_t$, formulated as follows: 
\begin{equation}\label{sec:crossentropy}
    \mathcal{L}(\bm X_t; \mathcal{W})=-\sum\limits_{k\in\mathcal{Y}}\Delta_{t,k}\cdot \log\left(\hat p(k\mid \bm X_t)\right).
\end{equation}
By substituting $\Delta_{t,k}$ for the ground-truth label indicator $\mathbbm{1}\{Y_t=k\}$, the loss function becomes an unbiased estimator of the traditional cross-entropy loss with full feedback $-\log\left(\hat p(Y_t\mid \bm X_t)\right)$ by following a similar derivation in \cref{thm:indicator}. This allows using information in those instances where the true label $Y_t$ is not explicitly available.

The estimated probability mass function $\hat p(k\mid \bm X_t)$ for each class $k$ is derived from the outputs of the neural network. Specifically, it is modeled by applying the softmax function to the logits $f^k_\mathcal{W}(\bm X_t)$ produced by the neural network:
\begin{equation}\label{eq:softmax}
    \hat p(k\mid \bm X_t):=\frac{\exp(f^k_\mathcal{W}(\bm X_t))}{\sum_{\tilde k\in\mathcal{Y}}\exp(f^{\tilde k}_\mathcal{W}(\bm X_t))}, ~k\in\mathcal{Y}.
\end{equation}
By integrating the estimator $\Delta_{t,k}$ with the softmax output, our model can update efficiently by optimizing the tailored loss function \eqref{sec:crossentropy} with stochastic gradient descent. It is important to note that one may employ other loss functions, such as the hinge loss in SVMs  \citep{kakade2008efficient}.

\cref{fig:CE} presents a clear visualization of the cross-entropy loss across three real datasets in the bandit feedback setting. It shows the model fitting performance with the softmax and uniform policies. The plots illustrate that during the model training phase, the softmax policy consistently achieves a more rapid reduction in loss compared to the uniform policy. This superior performance can be attributed to the context-aware nature of the softmax policy, which strategically pulls arms based on the specific context of each query. This approach not only leads to a higher frequency of accurate predictions but also ensures better utilization of data points, thereby enhancing the overall efficiency and effectiveness of the model training process.

\subsection{The Quantile of the Conformity Score}\label{sec:quantile}
To control the class-specific coverage, our approach leverages thresholds/quantiles associated with a given conformity score function $s(\bm X, k)$, such as softmax, APS \citep{romano2020classification}, or RAPS \citep{angelopoulos2021uncertainty} score (see the definitions in \cref{app:implent}). Particularly, the primary goal in this phase is to determine a $100\times\alpha\%$ quantile $\tau_{k}$ of the distribution of $s(\bm X, k)$. To this end, the traditional split conformal method \cite{papadopoulos2002inductive,lei2013distribution} involves partitioning available labeled data into training and calibration sets. However, in the online setting, since we only have access to a limited dataset at each iteration, split conformal may lead to two primary issues: (1) reduced data for model training, and (2) large prediction sets due to limited labeled calibration data \citep{ding2024class}. These two issues are further aggravated in the bandit feedback setting because only those data whose correct arms are pulled are considered labeled.

To overcome these challenges, we adaptively update a quantile estimate $\tau_{k}$ by utilizing the check loss function \citep{JMLR:v7:takeuchi06a,koenker1978regression,romano2019conformalized,gibbs2021adaptive} for quantile estimation:
$$\rho_{\alpha}(s, \tau)=(s-\tau)\cdot\bigl(\alpha-\mathbbm{1}\{s<\tau\}\bigr).$$ 
More concretely, a class-specific $100\times\alpha\%$ quantile $\tau_{k}, k\in\mathcal{Y}$ is obtained by solving the below optimization problem:
\begin{align}
    &~~~~\argmin_\tau \mathbb{E}\bigl[\rho_\alpha(s(\bm X, k), \tau)\mid Y = k\bigr]\notag\\
    &=\argmin_\tau\frac{\mathbb{E}\big[\mathbbm{1}\{Y=k\}\cdot\rho_\alpha(s(\bm X, k), \tau)\big]}{\mathbb{E}\big[\mathbbm{1}\{Y=k\}\bigr]}\notag\\
    &=\argmin_\tau\mathbb{E}\big[\mathbbm{1}\{Y=k\}\cdot\rho_\alpha(s(\bm X, k), \tau)\big], \label{eq:popQuantLoss}
\end{align}
where the second equality holds due to the fact that $\mathbb{E}\big[\mathbbm{1}\{Y=k\}\bigr]=\mathbb{P}(Y=k)$ does not \textcolor{mycol}{rely} on the quantile estimation. \textcolor{mycol}{Given that the true joint density function $p(\bm x, y)$ is unknown}, we instead employ a data-driven approach for quantile estimation: for each data point consider the loss
\begin{align}
\Delta_{t, k}\cdot\rho_\alpha(s(\bm X_t, k), \tau)\label{eq:empQuantLoss},
\end{align}
which is an empirical counterpart of the population loss \eqref{eq:popQuantLoss}. Consequently, $\tau_{k}$ can be dynamically updated through stochastic gradient descent by computing the gradient, $-\Delta_{t,k}\cdot\bigl(\alpha-\mathbbm{1}\{s(\bm X_t, k)< \tau\}\bigr)$, of  the weighted loss \eqref{eq:empQuantLoss}. The updated quantiles $\tau_{k}, k\in\mathcal{Y}$ are then applied as the thresholds for the upcoming data in the next iteration only. The complete process, including the model training and quantile estimation in an online learning context, is outlined in \cref{alg,fig:flowchart}.

\begin{algorithm}[!th]
  \caption{Bandit Conformal}\label{alg}
  \begin{algorithmic}[1]
  \REQUIRE Initialize weight matrices $\mathcal{W}^0$ and class-specific quantiles $\tau^0_{k} =0, k\in\mathcal{Y}$. Provide a score function $s^t(\cdot, \cdot)$\footnotemark, a policy $\pi_t$ and learning rates $\eta_1, \eta_2$.
  \FOR{$t=1, 2, 3, \cdots, T$}
    \STATE Learner receives a query $\bm X_t$
    \STATE Generates a prediction set for the query:  $$\widehat{\mathcal{C}}^{t-1}(\bm X_t):=\left\{k\in\mathcal{Y}: s^{t-1}(\bm X_t, k)\geq\tau^{t-1}_{k}\right\}$$
    \STATE Learner pulls an arm $A_t\sim\pi_t$, receives the feedback $\mathbbm{1}\{A_t=Y_t\}$, and computes $\Delta_{t,k}$
  
  \STATE Update the network weight matrices and quantiles: 
  \begin{equation*}
      \resizebox{\linewidth}{!}{\ensuremath{\begin{cases}
    \mathcal{W}^{t}=\mathcal{W}^{t-1}- \eta_1\nabla_\mathcal{W}\mathcal{L}(\bm X_t;\mathcal{W}^{t-1})\vspace{5pt}\\
    \tau^{t}_{k}= \tau^{t-1}_{k} + \eta_2\Delta_{t,k}\bigl(\alpha-\mathbbm{1}\{s^{t-1}(\bm X_t, k)< \tau^{t-1}_{k}\}\bigr)
  \end{cases}}}
  \end{equation*}
  \ENDFOR
\end{algorithmic}
\end{algorithm}
\footnotetext{We add the superscript $t$ on the score function to explicitly impress that it depends on the neural network updated up to $t$-th iteration. The same argument is applied to other notations.}

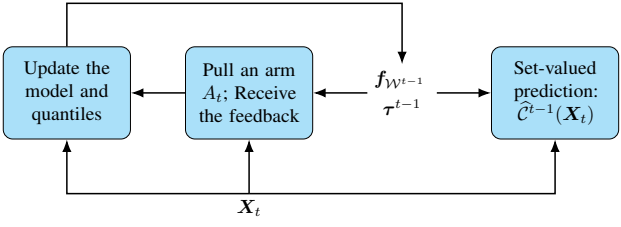
\begin{figure}[!ht]
    \centering
\resizebox{\linewidth}{!}{
\begin{adjustbox}{trim=0pt 3pt 0pt 3pt,clip}
\begin{tikzpicture}[
    block/.style={rectangle, draw, text width=6em, text centered, minimum height=17mm, node distance=1.cm, rounded corners=5pt
},
    line/.style={draw, -Latex, thick},
    line_no_arrow/.style={draw, thick},
]


    \node [block, fill=cyan!30] (updatemodel) {Update the model and quantiles};
    \node [block, fill=cyan!30, right=of updatemodel] (pullarm) {Pull an arm $A_t$; Receive the feedback};
    \node [align=center, right=of pullarm] (inputmodel) {$\bm f_{\mathcal{W}^{t-1}}$ \vspace{20pt} \\ [4pt] $\bm \tau^{t-1}$};
     \node [block, fill=cyan!30, right=of inputmodel] (setvalued) {Set-valued prediction: $\widehat{\mathcal{C}}^{t-1}(\bm X_t)$ };

    
    \node [above=of setvalued] (loct0) {};
    \node [above=of pullarm] (loct1) {};
    \node [above=of updatemodel] (loct2) {};
    \node [above=of inputmodel] (loctinput) {};
    
    \node [below=of setvalued] (locb0) {};
    \node [below=of pullarm] (input) {\(\bm X_t\)};
    \node [below=of updatemodel] (locb2) {};

    \path [line] (updatemodel) |- (loctinput.center) -- (inputmodel);
    
    \path [line] (input.north) -| (setvalued);
    \path [line] (input) -- (pullarm);
    \path [line] (input.north) -| (updatemodel);

    \path [line] (pullarm) -- (updatemodel);
    \path [line] (inputmodel) -- (setvalued);
    \path [line] (inputmodel) -- (pullarm);
\end{tikzpicture}
\end{adjustbox}}\vspace{-10pt}
    \caption{Flowchart of the online learning with bandit feedback. 
    Here $\bm \tau^{t-1}=(\tau^{t-1}_1, \cdots, \tau^{t-1}_{|\mathcal
    Y|})^\top$.}
    \label{fig:flowchart}
\end{figure}

When comparing with \citet{gibbs2021adaptive,gibbs2022conformal,zaffran2022adaptive,bhatnagar2023improved}, a critical aspect \textcolor{mycol}{differentiating} our method lies in the quantile updating process in addition to the model training in the bandit feedback context as elucidated in \cref{sec:banditCE}. In particular, the aforementioned studies predominantly work with the unweighted quantile estimation and require verification of whether the true label $Y_t$ falls in its prediction set $\widehat{\mathcal{C}}^{t-1}(\bm X_t)$ in their updating rules. This verification is typically achieved either by directly utilizing explicit label information or through multiple arm pulls until the true label is ascertained with absolute certainty. Such methodologies are not feasible in our setting for two primary reasons: (1) we lack direct access to the true label information, and (2) our framework does not permit multiple arm pulls for a single decision instance. In contrast, our approach (see the updating rule in \cref{alg}) involves computing the gradient of the weighted check loss \labelcref{eq:empQuantLoss} in the bandit feedback setting, which is an unbiased estimator of the gradient of unweighted check loss in the full feedback. This process is tailored to bandit feedback environments where each query allows only a single arm pull.

The below theorem implies the empirical coverage converges to the prescribed coverage.

\begin{theorem}\label{thm:cvg}
        Define the filtration $\mathcal{F}_t:=(\sigma(\bm X_t, Y_t)\times\sigma(\pi_t))\cup\mathcal{F}_{t-1} 
        $. Assume $\pi_t(k\mid \bm X_t)\geq c_k>0$ for all $t\in[T]$ and $\mathbb{E}[\frac{\mathbbm{1}\{Y_t=k\}}{\pi_t(k\mid \bm X_t)}\mid \mathcal{F}_{t-1}]=b^t_k$.      
        With probability at least $1-\delta$ \textcolor{mycol}{taken over all the randomness}, for all class $k\in\mathcal{Y}$, \cref{alg} yields the empirical coverage gap 
\begin{align}
    \text{CvgGap}_k&:=\biggl|\alpha-\frac{1}{T_k}\sum_{t=1}^T\mathbbm{1}\{Y_t=k\}\cdot\mathbbm{1}\{Y_t\not\in\widehat{\mathcal{C}}^{t-1}(\bm X_t)\}\biggr|\notag\\
    &\leq 
    \frac{\tau^{T}_{k}}{\eta_2 T_k}+\frac{\zeta_k(T,\delta/|\mathcal{Y}|)}{T_k}\notag,
\end{align}
where $\zeta_k(T,\delta)=\frac{2}{3c_k}\log\frac{2}{\delta}+\sqrt{2\log\frac{2}{\delta}\cdot \sum_{t=1}^Tb^t_k}$, and $T_k=\sum_{t=1}^T\mathbbm{1}\{Y_t=k\}$.
\end{theorem}

\cref{thm:cvg} implies the convergence rate of the class-specific coverage guarantee mainly depends on the learning rate $\eta_2$ and the sample size $T_k$ of class $k$. Besides the policy should be bounded strictly below by 0, the additional assumption on $\mathbb{E}[\frac{\mathbbm{1}\{Y_t=k\}}{\pi_t(k\mid \bm X_t)}\mid \mathcal{F}_{t-1}]$ further suggests that the policy should not overly underestimate the proportion of a class; otherwise the empirical coverage gap may increase.

\begin{color}{mycol}
To some extent, \cref{thm:cvg} ensures that the algorithm yields prediction sets with small sizes. This is because an algorithm with a large prediction set size often comes with inflated coverage, yet the theorem states that the empirical non-coverage must not deviate much away from the desired non-coverage of $\alpha$. In particular, \cref{thm:cvg} precludes the trivial case $\widehat{\mathcal{C}}^{t-1}(\boldsymbol X_t)=\mathcal{Y}$ for all $t\in[T]$.
\end{color}

The below corollary highlights the impact of different policies on the convergence rate.
 
\begin{corollary}\label{thm:finercvg}
    Assume the learning rate has the order $\eta_2=\mathcal{O}(T^{-1/2})$. (1) If the policy $\pi_t$ aligns with the Bayes posterior probability, i.e., $\pi_t(k\mid \bm X_t)=\mathbb{P}(Y_t=k\mid \bm X_t)$, then we have $\mathbb{E}[\frac{\mathbbm{1}\{Y_t=k\}}{\pi_t(k\mid \bm X_t)}\mid \mathcal{F}_{t-1}]= b^t_k\leq 1$, and hence
    $\text{CvgGap}_k= \mathcal{O}(\frac{\sqrt{T}}{T_k}).$ (2) If the policy is the uniform distribution, i.e., $\pi_t(k\mid \bm X_t)=\frac{1}{|\mathcal{Y}|}$, then $b^t_k\leq |\mathcal{Y}|p_k$ (here $p_k$ denotes the prior probability of class $k$), and hence 
    $\text{CvgGap}_k= \mathcal{O}(\frac{\sqrt{T|\mathcal{Y}|p_k}}{T_k}).$
\end{corollary}

\cref{thm:finercvg} implies a convergence rate of 
$\text{CvgGap}_k=\mathcal{O}(T^{-1/2})$ when the learning rate $\eta_2=\mathcal{O}(T^{-1/2})$ and sample size $T_k=\mathcal{O}(T),~k\in\mathcal{Y}$ under both Bayes posterior probability and uniform probability policies. In our experiments, due to the lack of access to the precise data distribution, we instead use the softmax policy, i.e., $\pi_t(k\mid \bm X_t)=\hat{p}(k\mid \bm X_t)$ as defined in \eqref{eq:softmax}, to estimate the Bayes posterior probability. As noted by \citet{tibshirani2019conformal}, there are alternative methods for probability estimation, such as moment matching and Kullback-Leibler Divergence minimization. We refer to related work \cite{sugiyama2012density} for a comprehensive review.

\begin{theorem}\label{thm:checkReg}
Let $p_k$ be the prior probability of class $k\in\mathcal{Y}$, and $\tau_k^*=\argmin_\tau \frac{1}{T}\sum_{t=1}^T\mathbbm{1}\{Y_t=k\}\rho_{\alpha}(s^{t-1}(\bm X_t), \tau)$ be the quantile estimate using all the data instances. Define the empirical regret associated with the check loss in the bandit feedback setting as $\text{Reg}_{k, \rho_\alpha}(T):=\frac{1}{T}\sum_{t=1}^T\Delta_{t,k}\rho_{\alpha}(s^{t-1}(\bm X_t), \tau^{t-1}_{k})-\frac{1}{T}\sum_{t=1}^T\mathbbm{1}\{Y_t=k\}\rho_{\alpha}(s^{t-1}(\bm X_t), \tau_k^*)$. By choosing $\eta_2=\tau^*_kp_k^{1/2}\bigl(\sum_{t=1}^T\mathbb{E}\bigl[\frac{\mathbbm{1}\{Y_t=k\}}{\pi^2_t(k\mid \bm X_t)}\bigr]\bigr)^{-1/2}$, \cref{alg} yields an expected regret
\begin{align}
\mathbb{E}[\text{Reg}_{k, \rho_\alpha}(T)]&\leq 
\frac{\tau^*_k}{T}\sqrt{p_k\sum\limits_{t=1}^T\mathbb{E}\biggl[\frac{\mathbbm{1}\{Y_t=k\}}{\pi^2_t(k\mid \bm X_t)}\biggr]}.\notag
\end{align}
\end{theorem}
The above expectation is taken over \textcolor{mycol}{over all the randomness, including the data and algorithm}. Note that $\tau_k^*$ is bounded, and hence the upper bound converges to 0.

\begin{corollary}\label{thm:finechk_reg}
For the uniform policy and an appropriately chosen $\eta_2$ as specified in \cref{thm:checkReg}, the expected regret $\mathbb{E}[\text{Reg}_{k, \rho_\alpha}(T)]\leq\frac{\tau^*|\mathcal{Y}|p_k}{\sqrt{T}}$.
\end{corollary}

\cref{thm:finechk_reg} indicates that the expected regret adheres to a theoretical convergence rate of $\mathcal{O}(T^{-1/2})$, under the condition that the learning rate $\eta_2=\mathcal{O}(T^{-1/2})$ (it can be achieved when the policy is bounded strictly below by 0). This condition aligns with the findings in \cref{thm:finercvg}. 

\textcolor{mycol}{Both \cref{thm:checkReg,thm:finechk_reg} provide theoretical guarantees for the convergence behavior of \cref{alg} in a parametric rate, indicating its potential effectiveness. This result shows that there exists such a learning rate $\eta_2$ leading to an optimal convergence rate. How to practically obtain such a precise learning rate is a challenging problem.} In practice, as discussed in the work of \citet{gibbs2021adaptive}, the chosen value of $\eta_2$ leads to two distinct scenarios.
A larger value of $\eta_2$ may lead to unstable quantile estimations, causing oscillations in prediction set sizes. Over time, this could result in increasingly larger prediction sets in the online learning process. Conversely, a smaller value of $\eta_2$ slows the convergence rate of the coverage, necessitating more iterations to achieve desired coverage levels. 

\begin{algorithm}[!ht]
  \caption{Bandit Conformal with Experts}\label{alg:expert}
  \begin{algorithmic}[1]
  \REQUIRE Initialize weight matrices $\mathcal{W}^0$, class-specific quantiles $\tau^0_{j,k}=0$, and experts weights $\omega^0_{j, k}=1,~j\in[J], k\in\mathcal{Y}$. A score function $s^t(\cdot, \cdot)$, a policy $\pi_t$ and learning rates $\eta_1, \eta_{2,j}, j\in[J]$.
  \FOR{$t=1, 2, 3, \cdots, T$}
    \STATE Learner receives a query $\bm X_t$
    \STATE Generates a prediction set for the query:  
    \begin{equation*}
    \widehat{\mathcal{C}}^{t-1}(\bm X_t):=\left\{k\in\mathcal{Y}: s^{t-1}(\bm X_t, k)\geq\bar{\tau}^{t-1}_{k}\right\},
    \end{equation*}
    where $\bar{\tau}^{t-1}_{k}=\sum_{j}\omega_{j,k}^{t-1}\tau^{t-1}_{j, k}/\sum_i\omega_{i,k}^{t-1}$
    \STATE Learner pulls an arm $A_t\sim\pi_t$, receives the feedback $\mathbbm{1}\{A_t=Y_t\}$, and computes $\Delta_{t,k}$
  \STATE Update all weights and quantiles: 
  \begin{equation*}
  \resizebox{\linewidth}{!}{\ensuremath{\begin{cases}
    \mathcal{W}^{t}=\mathcal{W}^{t-1}- \eta_1\nabla_\mathcal{W}\mathcal{L}(\bm X_t;\mathcal{W}^{t-1})\vspace{5pt}\\
    \tau^{t}_{j, k} = \tau^{t-1}_{j, k} + \eta_{2,j}\Delta_{t,k}\bigl(\alpha-\mathbbm{1}\{s^{t-1}(\bm X_t, k)< \tau^{t-1}_{j, k}\}\bigr)\vspace{5pt}\\
    \omega^t_{j, k} = \exp(-\frac{1}{\sqrt{t+1}}\sum\limits_{t'\leq t}\Delta_{t',k}\cdot \rho_\alpha(s^{t'-1}(\bm X_{t'}, k), \tau^{t'-1}_{j,k}))
  \end{cases}}}
  \end{equation*}
  \ENDFOR
\end{algorithmic}
\end{algorithm}

To mitigate the above limitation due to the choice of $\eta_2$, we draw inspiration from the adaptive control method in its full feedback setting \citep{zaffran2022adaptive}. We introduce an alternative algorithm, Bandit Conformal with Experts (outlined in \cref{alg:expert}), which eliminates the need for manual tuning of $\eta_2$. Specifically, given a grid of learning rate values $\eta_{2,j}, j\in[J]$, it employs an ensemble methodology to aggregate estimated quantiles associated with $\eta_{2,j}$'s based on past performance. The guiding principle is that as the accumulated check loss decreases, the attention placed on the corresponding estimated quantile grows.

\cref{thm:checkReg_expert} below shows that the aggregated quantile through the experts converges to the optimal quantile estimate among the experts. Specifically, an increase in the number of experts, while maintaining the order $J=\mathcal{O}(1)$, can enhance the chance of achieving an improved learning rate, along with more accurate quantile estimations. This finding underscores the importance of expert integration in improving algorithmic performance if one has no prior idea of the optimal learning rate.

\begin{theorem}\label{thm:checkReg_expert}
Consider $\bar\tau^{t-1}_k$ as the aggregated quantile across $J$($\geq2$) experts as defined in \cref{alg:expert}, and the same $c_k$ defined in \cref{thm:cvg}. 
Then, \cref{alg:expert} yields 
\begin{align}
&\frac{1}{T}\sum_{t=1}^T\Delta_{t,k}\rho_{\alpha}(s^{t-1}(\bm X_t), \bar\tau^{t-1}_k)\notag\\
&-\min_{j\in[J]}\frac{1}{T}\sum_{t=1}^T\Delta_{t,k}\rho_{\alpha}(s^{t-1}(\bm X_t), \tau_{j,k}^{t-1})\notag \\
&\leq \frac{1}{4c_k^2\sqrt{T}}+\frac{2\ln J}{\sqrt{T}}.\notag
\end{align}
\end{theorem}
Here the assumption for $c_k$ is reasonably flexible as it can be achieved through the policy design.

Notice that, theoretically, the optimal choice of learning rate should vary depending on the class as indicated in \cref{thm:checkReg}. However, for the ease of practical implementation, the same value of $\eta_2$ (or $\eta_{2,j}$) is applied across all classes.

\section{Experiments}\label{sec:expir}
\paragraph{Set-up:}  To assess the effectiveness of our proposed approach, we employ the ResNet50 architecture \citep{he2016deep} for model fitting. Our experimental setup includes the CIFAR10, CIFAR100 (with 20 coarser labels), and SVHN datasets, each undergoing 5 replications. Consistently throughout the study, we maintain a non-coverage rate $\alpha=0.05$. For computational efficiency, the model training is performed on data batches of size 256, utilizing the ADAM optimizer with a learning rate of $\eta_1=10^{-4}$ in the model training phase. The entire online learning process spans $T=6000$ iterations around. We evaluate online classification performance using three score functions: softmax, APS, and RAPS (see their definition in \cref{app:implent}) for both the softmax policy and the uniform policy.

\paragraph{Metrics:} To examine the performance during online prediction for $t\in[T]$, we report both the minimum and maximum accumulative coverage, defined as:
\begin{equation*}
\begin{split}
\text{Acum\_cvg\_min}(t) &= \min_{k\in\mathcal{Y}}\text{Acum\_cvg}(t, k),\\
\text{Acum\_cvg\_max}(t) &= \max_{k\in\mathcal{Y}}\text{Acum\_cvg}(t, k),
\end{split}
\end{equation*}
where $\text{Acum\_cvg}(t, k)$ is defined as
$$\begin{aligned}
   \frac{\sum_{s=1}^t\sum_{\bm X_i\in\mathcal{B}_s}\mathbbm{1}\{Y_i=k~\&~ Y_i\in \widehat{\mathcal{C}}^{t-1}(\bm X_i)\}}{\sum_{s=1}^t\sum_{\bm X_i\in\mathcal{B}_s}\mathbbm{1}\{Y_i=k\}},
\end{aligned}$$
with $\mathcal{B}_s$ representing the batch of the dataset at time point $s$. We include the accumulative prediction set size,
\begin{equation*}
\text{Acum\_size}(t)=\frac{\sum_{s=1}^t\sum_{\bm X_i\in \mathcal{B}_{s}}|\widehat{\mathcal{C}}^{t-1}(\bm X_i)|}{\sum_{s=1}^t|\mathcal{B}_s|},
\end{equation*}
to assess the informativeness of the set-valued classification.

\begin{figure}[!ht]
    \centering
     \includegraphics[width=1\linewidth]{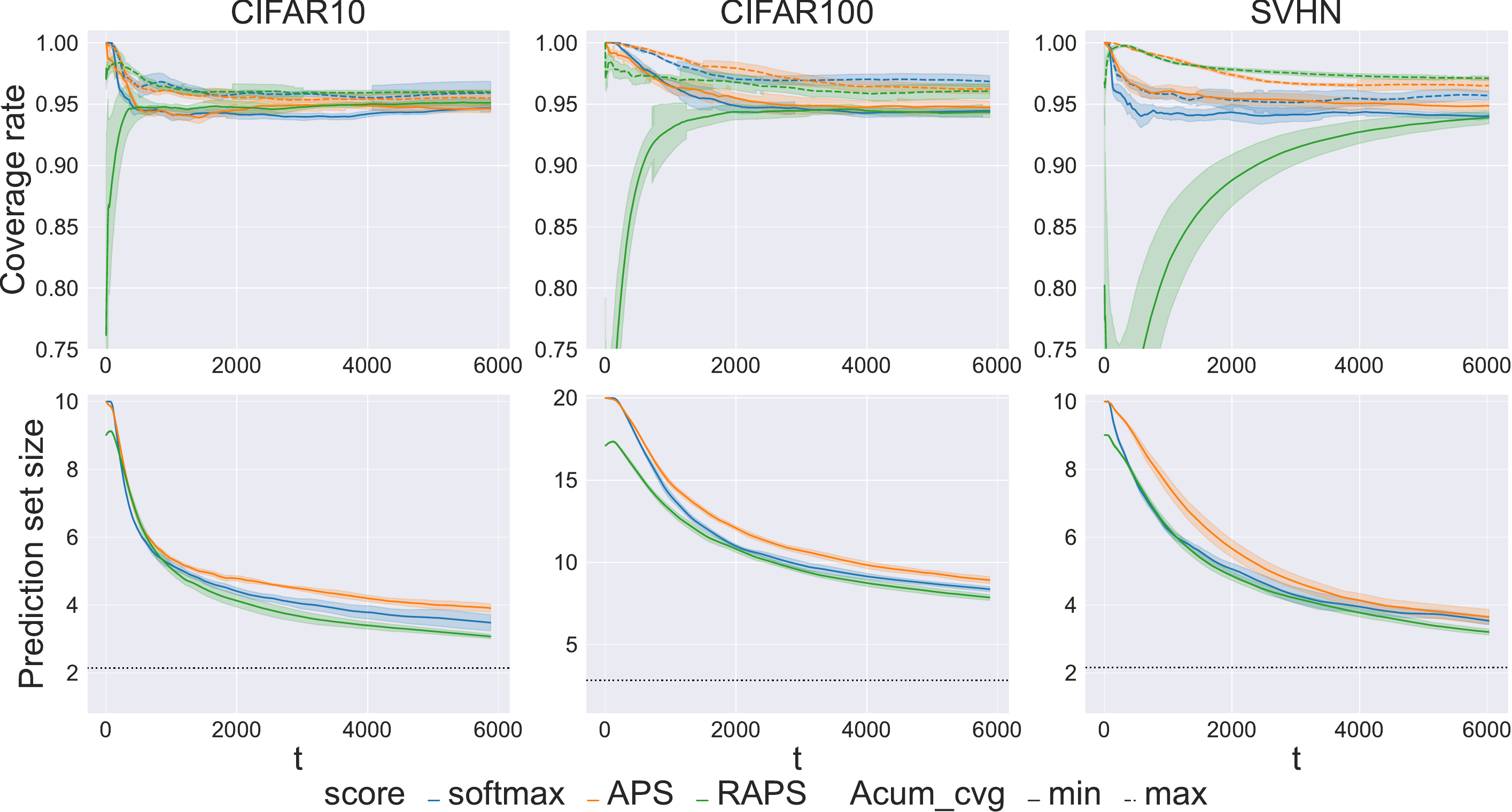}\vspace{-15pt}
    \caption{Performances under \cref{alg} with softmax policy. \textcolor{mycol}{The black dotted lines in the bottom panel denote the oracle performance of the model with access to the full labels.}}\label{fig:alg1softmax}\vspace{-7pt}
\end{figure}

\begin{figure}[!ht]
    \centering
     \includegraphics[width=1\linewidth]{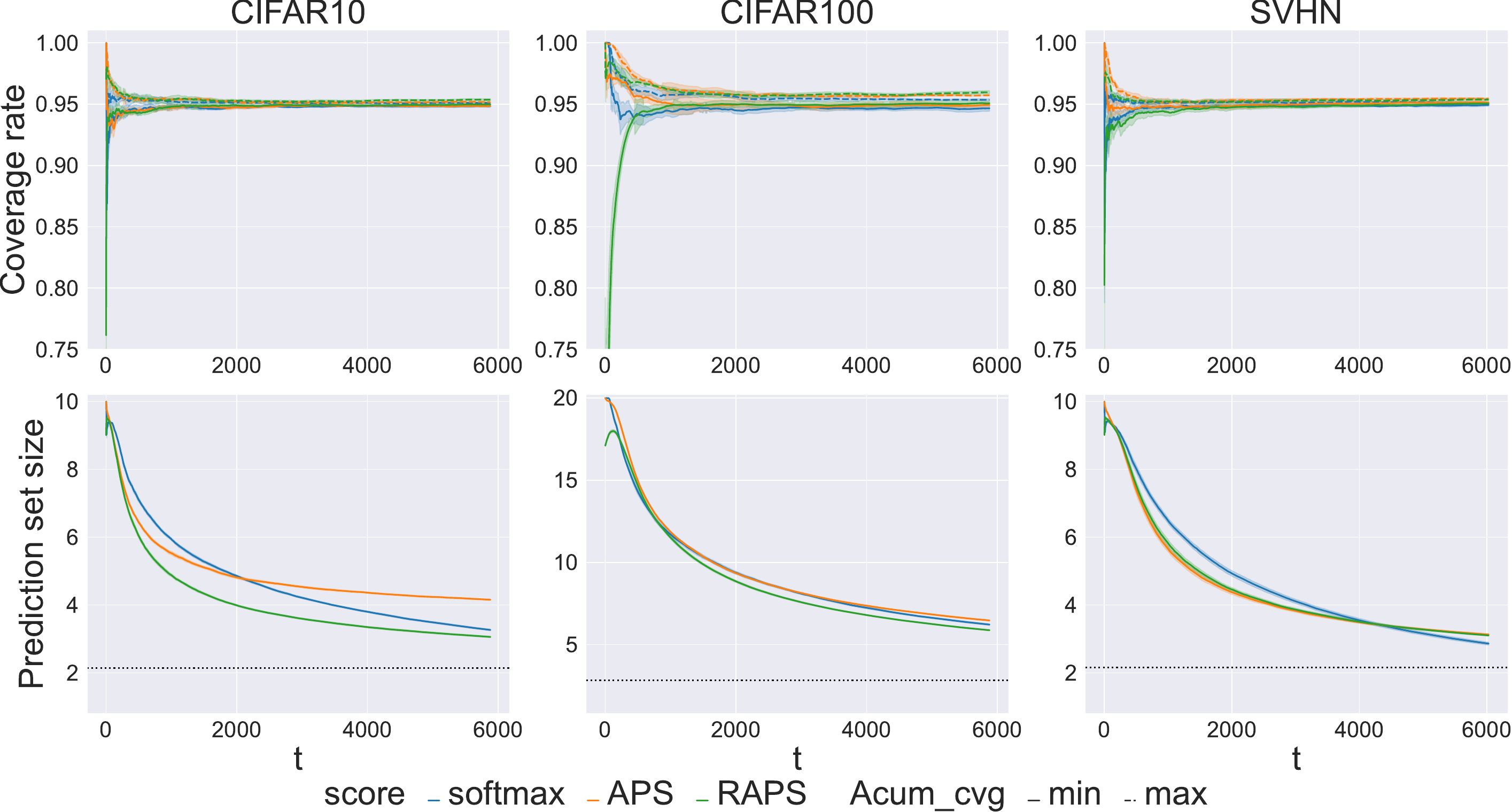}\vspace{-15pt}
    \caption{Performances under \cref{alg} with uniform policy. \textcolor{mycol}{The black dotted lines in the bottom panel denote the oracle performance of the model with access to the full labels.}}\label{fig:alg1uniform}\vspace{-7pt}
\end{figure}

\paragraph{Results:} \cref{fig:alg1softmax,fig:alg1uniform} present the set-valued classification with BCCP in the bandit feedback setting under softmax and uniform policies, respectively. The black dotted lines in the bottom panel for each figure denote the final result of a network after sufficiently many iterations with access to the full labels and the usage of the RAPS score function. As the number of iterations increases, the top panels in \cref{fig:alg1softmax,fig:alg1uniform} reveal that \cref{alg} effectively approaches the prescribed class-specific coverage of 95\%.  Additionally, the bottom panels in these figures indicate a trend towards smaller prediction sets.

The choice of learning rate $\eta_1$ indeed affects the performance of the model training phase and hence the subsequent quantile estimation. However, in our study, we mainly focus on the role of $\eta_2$ instead of particularly optimizing for $\eta_1$. For example, the CIFAR100 experiments utilizing the softmax policy and softmax score are presented with a fine-tuned $\eta_2=5\times 10^{-4}$ (see the tuning strategy and sensitivity studies about $\eta_2$ in \cref{sec:sens}). As discussed below \cref{thm:finechk_reg}, an inappropriate selection of the hyper-parameter $\eta_2$ can result in enlarged prediction set sizes or prolonged convergence times, which hinders the practical applicability of \cref{alg} in more dynamic settings.

\begin{figure}[!th]
    \centering
     \includegraphics[width=1\linewidth]{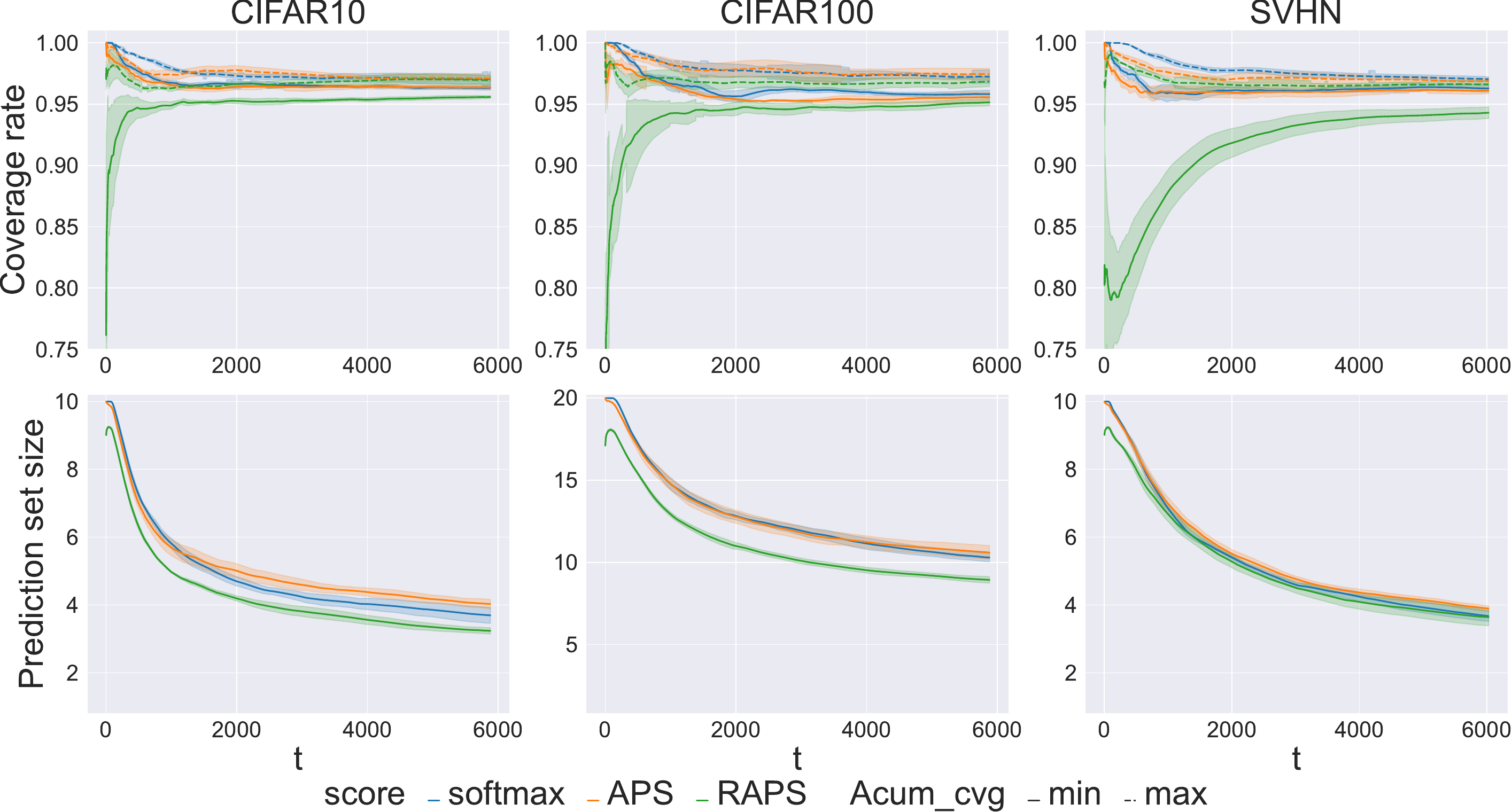}\vspace{-15pt}
    \caption{Performances under \cref{alg:expert} with softmax policy.}\label{fig:alg2softmax}
\end{figure}

\begin{figure}[!th]
    \centering
     \includegraphics[width=1\linewidth]{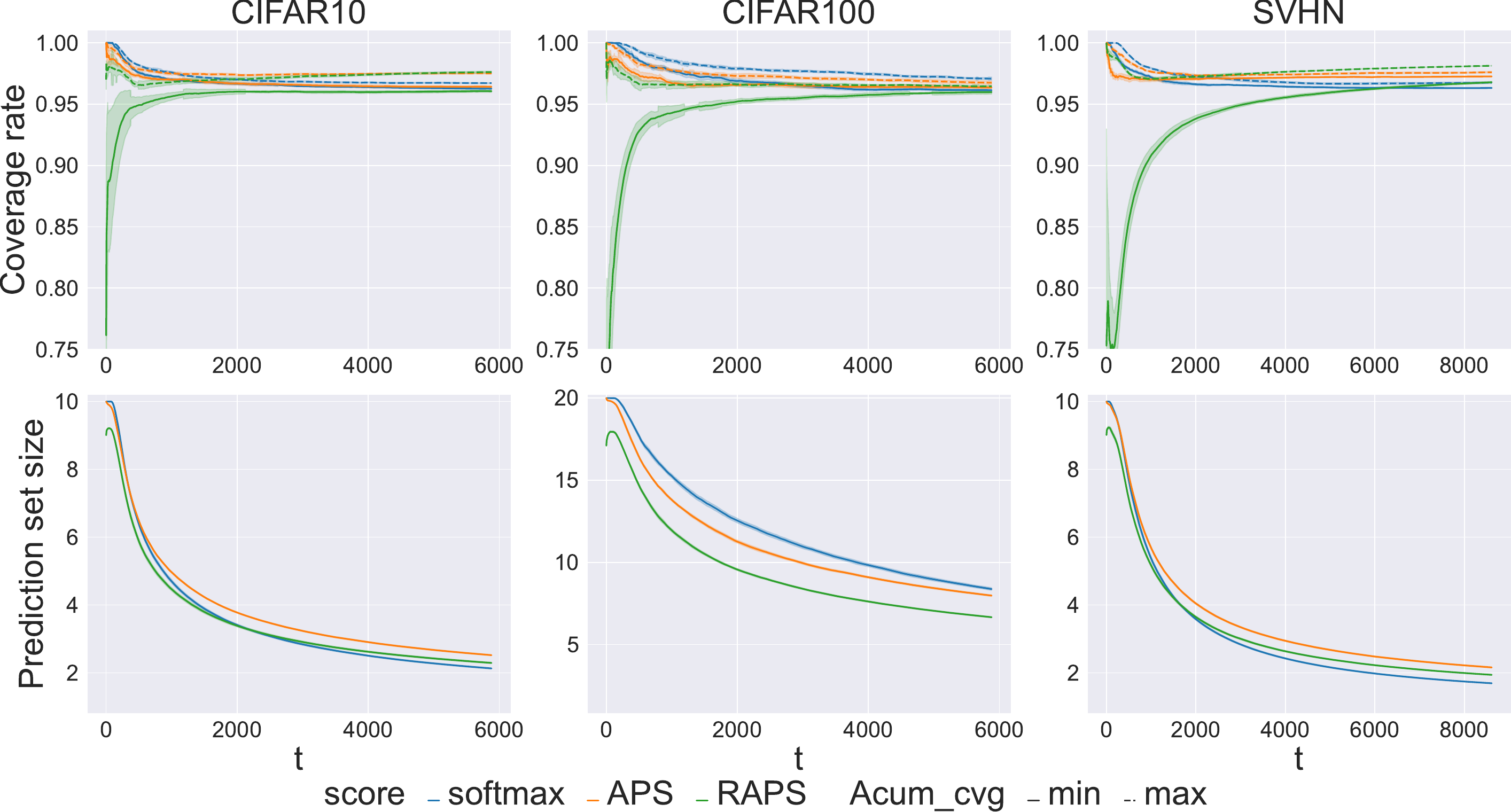}\vspace{-15pt}
    \caption{Performances under \cref{alg:expert} with uniform policy.}\label{fig:alg2uniform}
\end{figure}

To address this limitation, in this study, we employed a range of learning rate values, i.e., [0.1, 0.01, 0.001, 0.0001], through an expert-based approach in \cref{alg:expert}. The results are shown in \cref{fig:alg2softmax,fig:alg2uniform}. Notably, while using the softmax policy, the results from \cref{fig:alg2softmax} indicate that the prediction set sizes from \cref{alg:expert} are only marginally larger compared to those from \cref{alg} with carefully tuned $\eta_2$. With the uniform policy, \cref{alg:expert} demonstrates more efficient performance, yielding smaller prediction sets for the CIFAR10 and SVHN datasets. Notably, the RAPS score function outperforms the other scores in producing smaller prediction sets on the dataset when there are many classes, i.e., CIFAR100.

\section{Conclusion}\label{sec:conclusion}
In this article, we extend Conformal Prediction to the framework of online bandit feedback, where a learner is only told whether or not a pulled arm is correct in a dynamic multi-class classification problem. We make use of an unbiased indicator function estimation of the ground truth to overcome the incomplete information in the feedback, allowing the proposed Bandit Class-specific Conformal Prediction (BCCP) to effectively make set-valued inferences and adaptively fit the model accordingly. Particularly, the indicator function estimation allows us to utilize stochastic gradient descent to efficiently achieve the quantile estimation instead of the traditional split conformal, which requires sufficient labeled calibration data and might not be realistic in the setting of bandit feedback. Theoretically, we show the $\mathcal{O}(T^{-1/2})$ convergence rate for both the coverage guarantee and the regret of the check loss under certain conditions. Empirically, the experiments conducted on three datasets with three score functions and two policies demonstrated the effectiveness of BCCP. 

Our research opens several promising avenues for future exploration. One potential direction is the investigation of alternative indicator function estimations or policy designs that could offer improved theoretical or empirical performance. Additionally, refining the coverage guarantee within specific fixed-size time windows \citep{bhatnagar2023improved} instead of the full-time horizon in our work could further bolster the reliability of BCCP over different time scales. Moreover, expanding the scope of BCCP to address challenges such as covariate shift \cite{tibshirani2019conformal} and semantic shift \cite{JMLR:v24:23-0712} could significantly broaden its applicability. 

In conclusion, our work not only contributes a novel and provable solution to the problem of online multi-class classification with bandit feedback but also sets another new direction in conformal prediction. It opens up possibilities for real-world applications and lays a foundation for further research domains.

\begin{color}{mycol}
\section*{Impact Statements}
This paper presents work that aims to advance the field of Machine Learning. There are many potential societal consequences of our work, none of which we feel must be specifically highlighted here.
\end{color}


\bibliography{example_paper}
\bibliographystyle{icml2024}

\clearpage

\appendix
\onecolumn

\section{Conformity Scores}\label{app:implent}
Let $\hat p(k\mid \bm X), k\in\mathcal{Y}$ as defined in \eqref{eq:softmax} be the estimated posterior probability based on the neural network $\bm f_{\mathcal{W}}(\bm X)$. Thus, for the test point $(\bm X, Y)$,
the softmax score is defined as $$s(\bm X, k)=\hat p(k\mid \bm X).$$
Sort the estimated posterior probabilities $\hat p(k\mid \bm X), k\in\mathcal{Y}$ with the ascending order such that $\hat p(k_1\mid \bm X)\leq \hat p(k_2\mid \bm X)\leq \cdots\leq \hat p(k_{|\mathcal{Y}|}\mid \bm X).$ Additionally, denote $r$ as the index such that $k_r=Y$. Thus, the APS score is defined as
$$s(\bm X, k)=1-\sum_{l=1}^{r-1}\hat p(k_l\mid \bm X)-U \cdot \hat p(k_r\mid \bm X),$$ where $U$ is a random variable sampled from the uniform distribution on the interval $[0, 1]$. 

Let $k_{reg}$ be the number above which the prediction set size will be penalized with the penalty $\lambda$. Thus, the RAPS is defined as
$$s(\bm X, k)=1-\sum_{l=1}^{r-1}\hat p(k_l\mid \bm X)-U \cdot \hat p(k_r\mid \bm X)-\lambda\cdot[r - k_{reg}]_+,$$
where $[\cdot]_+=\max(0, \cdot)$. By following the similar routine in \citet{ding2024class}, in our experiments, we pick $\lambda=0.01$ and $k_{reg}=5$ for CIFAR100 while $k_{reg}=1$ for the remaining two less difficult datasets.

\section{Extra Studies on $\eta_2$}\label{sec:sens}

\begin{figure}[!b]
    \centering
     \includegraphics[width=0.95\linewidth]{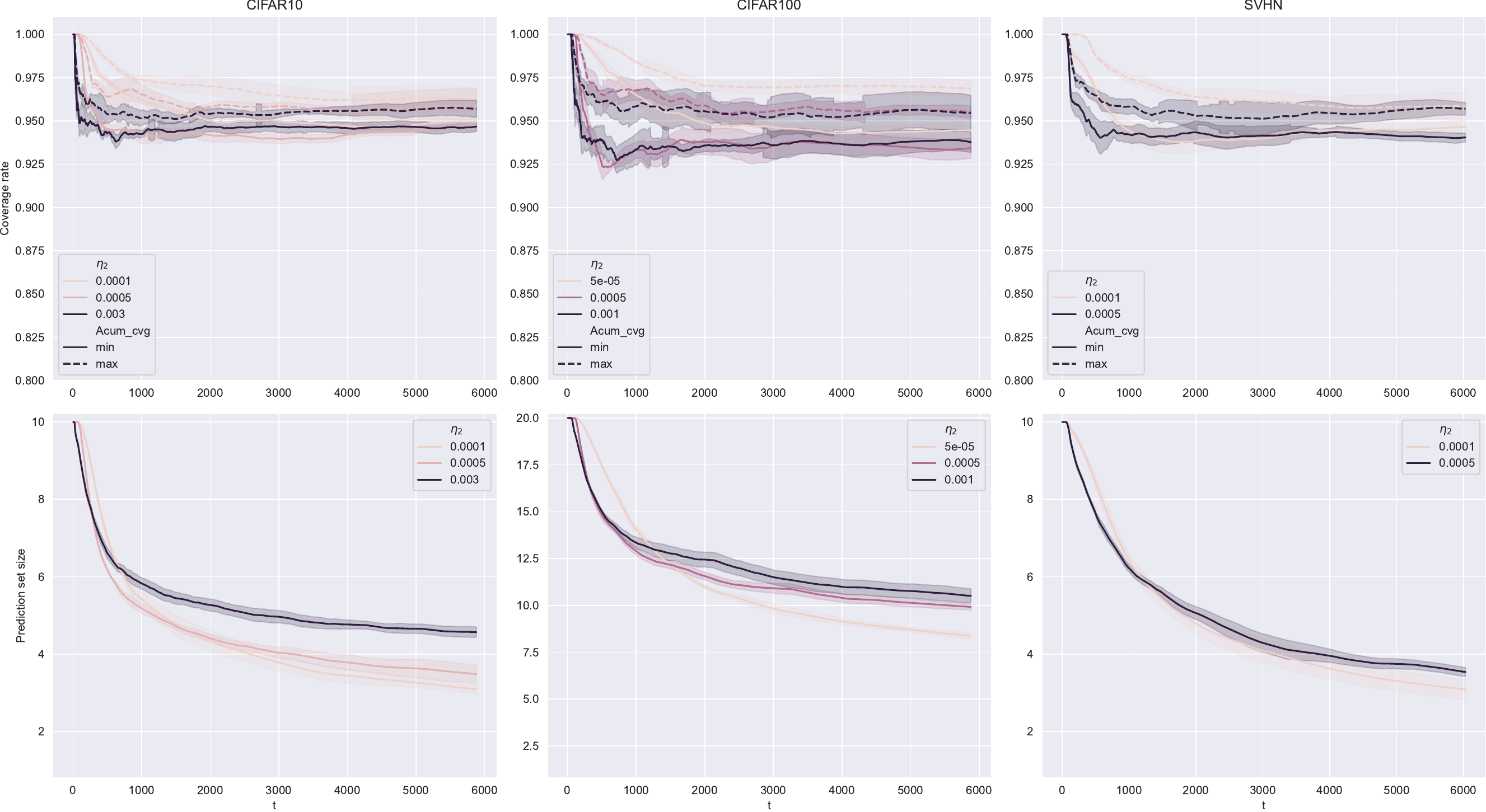}\vspace{-10pt}
    \caption{Performances under \cref{alg} with softmax policy and softmax score.}\label{fig:softmax_softmax}
\end{figure}

In our study, we adopt the learning rate tuning approach as described by \citet{gibbs2021adaptive}, selecting a value that ensures a stable learning trajectory characterized by a balance between smaller prediction set sizes and satisfactory coverage convergence. However, this tuning strategy presents challenges in practical applications. Specifically, different datasets require distinct optimal learning rate values, and identifying these values through manual tuning is both time-consuming and less adaptive. To illustrate these challenges, we conducted sensitivity analyses on the impact of varying $\eta_2$ in \cref{alg}. These studies underscore the limitations of manually tuning a single $\eta_2$ value.

Our findings, presented in \cref{fig:softmax_softmax,fig:uniform_softmax,fig:softmax_RAPS,fig:uniform_RAPS}, explore the sensitivity of the learning rate $\eta_2$. We observed that a higher $\eta_2$ value accelerates coverage control, as indicated by the darker lines in the top panels of each figure. However, this generally comes at the cost of enlarged prediction set sizes, evident from the darker lines in the bottom panels of the figures. Moreover, the prediction set size shows considerable sensitivity to variations in $\eta_2$. This highlights the practical limitations of \cref{alg} and underscores the necessity of implementing \cref{alg:expert}, which utilizes an expert-based method to aggregate results across multiple learning rates $\eta_{2,j}, j\in[J]$. This approach not only addresses the challenges of manual tuning but also enhances the algorithm's adaptability and effectiveness across diverse datasets.

\begin{figure}[!ht]
    \centering
     \includegraphics[width=0.95\linewidth]{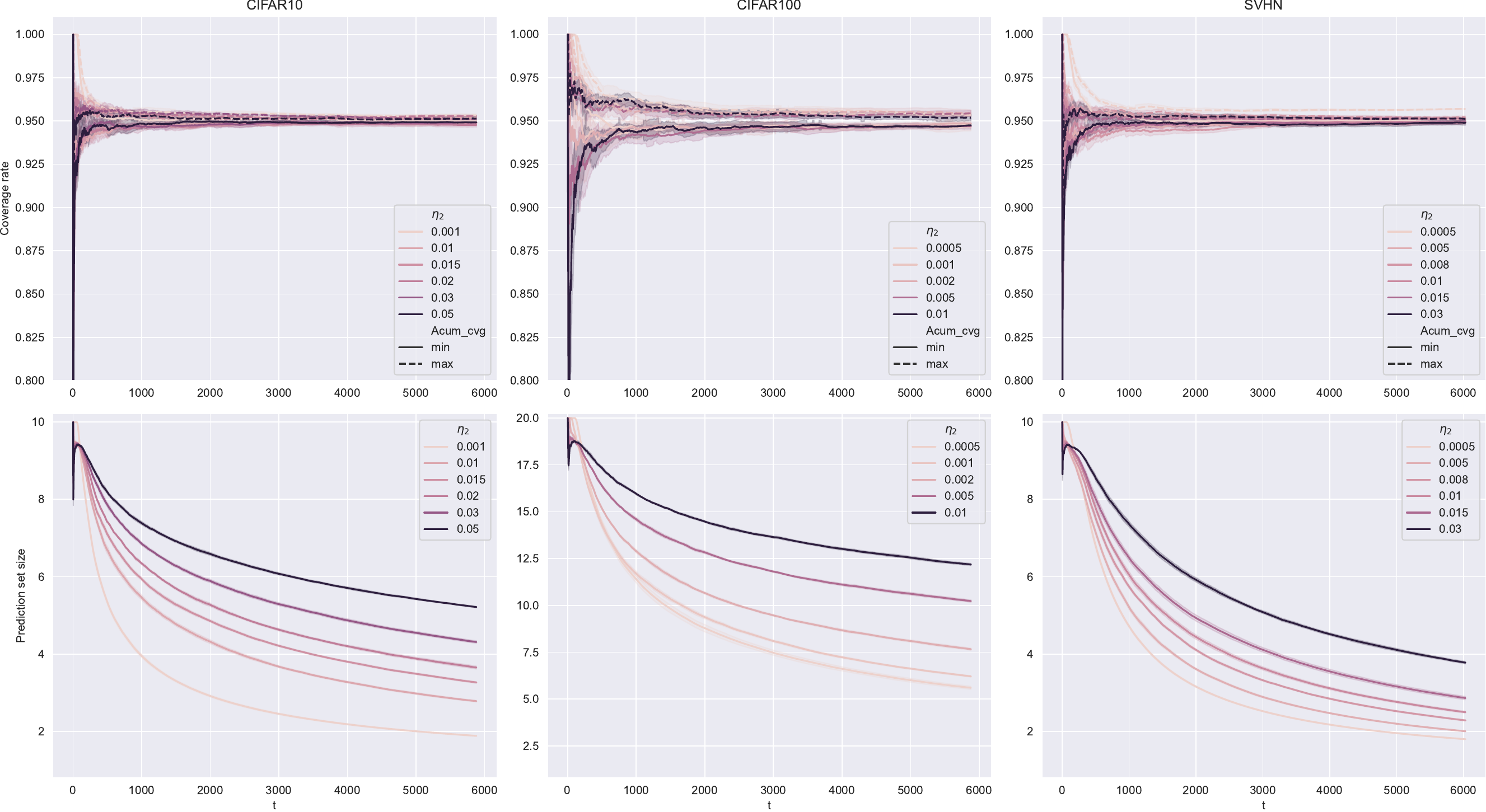}\vspace{-10pt}
    \caption{Performances under \cref{alg} with uniform policy and softmax score.}\label{fig:uniform_softmax}
\end{figure}

\begin{figure}[!ht]
    \centering
     \includegraphics[width=0.95\linewidth]{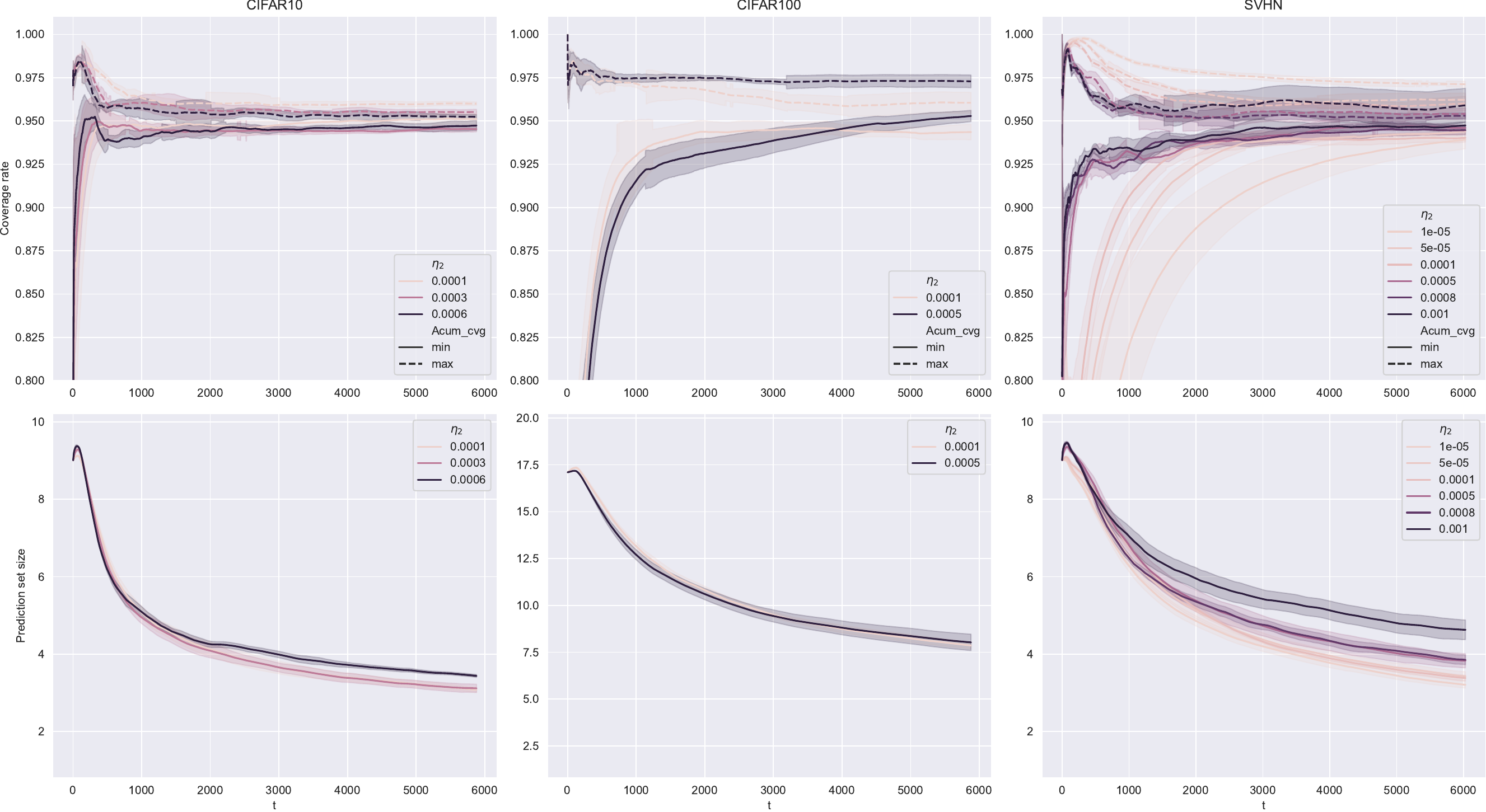}\vspace{-10pt}
    \caption{Performances under \cref{alg} with softmax policy and RAPS score.}\label{fig:softmax_RAPS}
\end{figure}

\begin{figure}[!ht]
    \centering
     \includegraphics[width=0.95\linewidth]{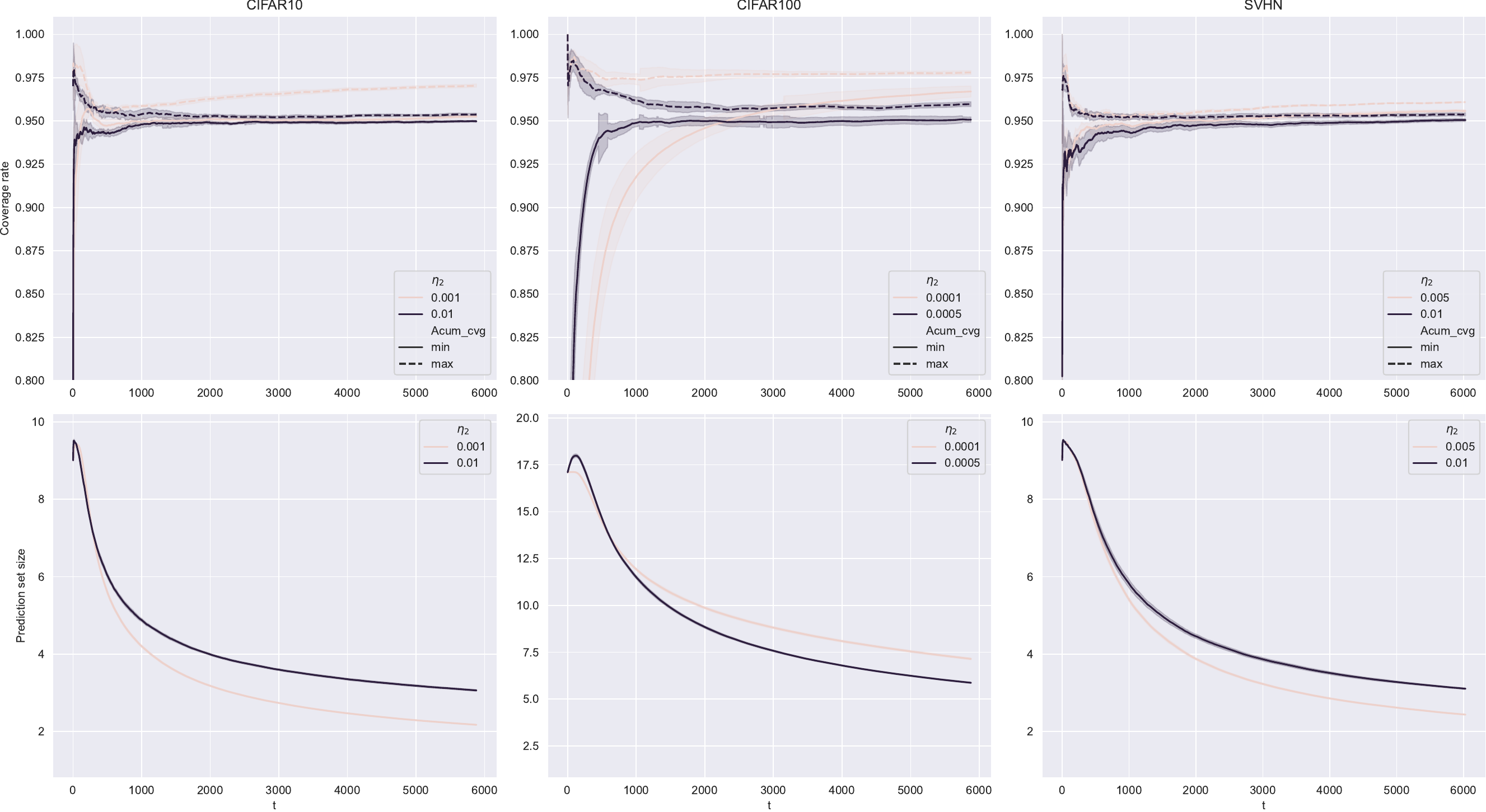}\vspace{-10pt}
    \caption{Performances under \cref{alg} with uniform policy and RAPS score.}\label{fig:uniform_RAPS}
\end{figure}

\begin{color}{mycol}

\begin{figure*}[!b]
    \centering
     \includegraphics[width=0.95\linewidth]{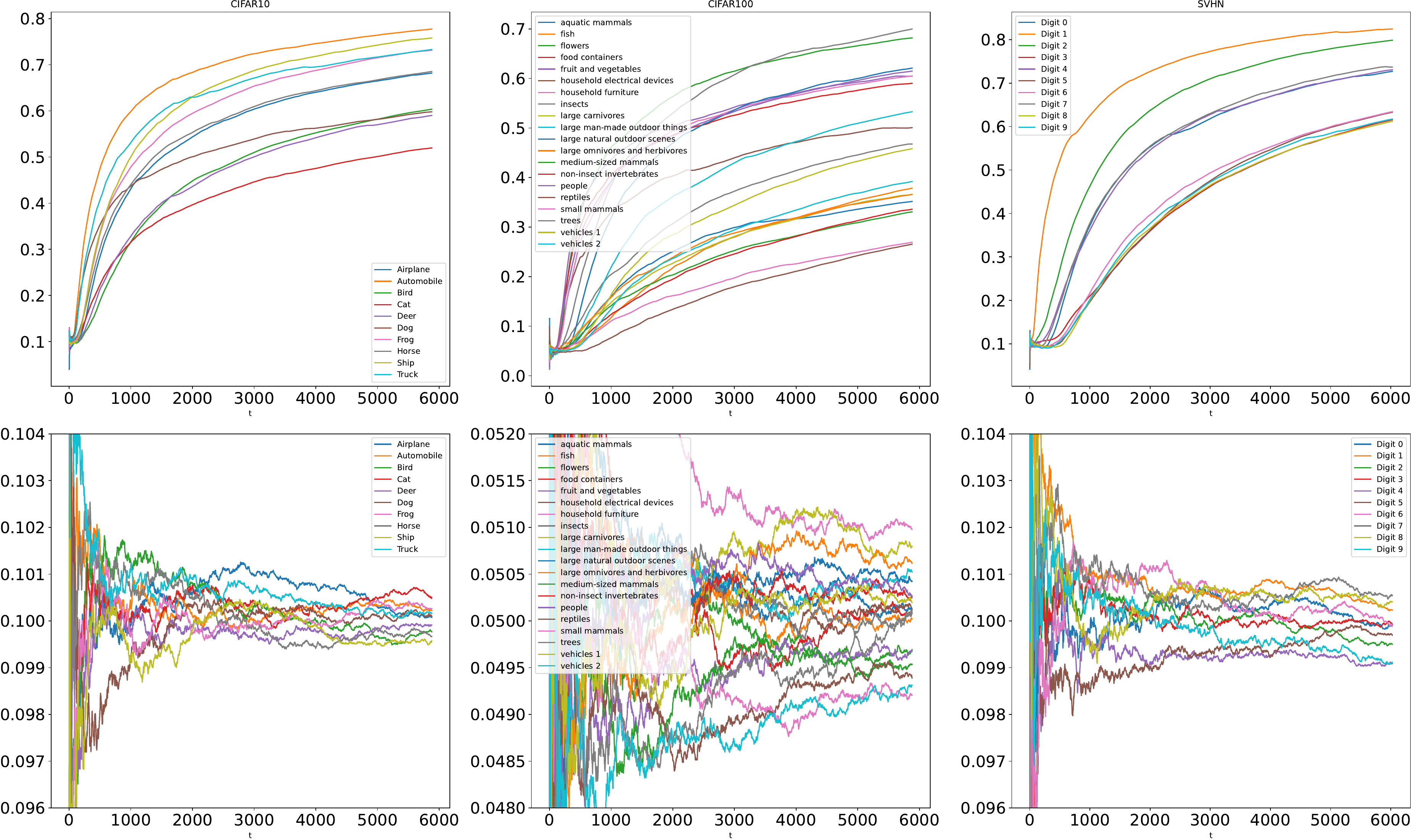}\vspace{-10pt}
    \caption{Proportion of correctly pulled arm with RAPS score under softmax (top) and uniform (bottom) policy.}\label{fig:corr_arm}
\end{figure*}

\section{Discussion of Policy $\pi_t$} 

In this section, for each class, we show the effectiveness of different policies on the correctness of arm pulling, i.e., $\mathbb{P}(A_t=Y_t\mid Y_t=k), ~k\in\mathcal{Y}$. In \cref{fig:corr_arm}, under the softmax policy (top panel) and the uniform policy (bottom panel), we report the accumulative performance of arm pulling for each class, i.e., 
$$\frac{\sum_{s=1}^t\sum_{\bm X_i\in\mathcal{B}_s}\mathbbm{1}\{A_i=k\}}{\sum_{s=1}^t\sum_{\bm X_i\in\mathcal{B}_s}\mathbbm{1}\{Y_i=k\}},~k\in\mathcal{Y}.$$

Due to the usage of context $\bm X_i$ in each batch $\mathcal{B}_s, s\leq t$, softmax policy leads to higher accuracy for arm pulling. In contrast, the uniform policy's correctness is close to $\frac{1}{|\mathcal{Y}|}$. These behaviors align with the one in cross-entropy loss minimization in \cref{fig:CE}, where the softmax policy quickly decreases the loss compared to the uniform policy. On the other hand, when it comes to the performance of set-valued classification in \cref{fig:alg1softmax,fig:alg2softmax}, the uniform policy both converges faster to the desired coverage rate and gets slightly smaller prediction sets on average than the softmax policy.

The above interesting phenomenon may mirror the exploration-exploitation dilemma in reinforcement learning. Specifically, the softmax policy capitalizes on more known information characterized by $\hat p(k\mid \boldsymbol X_t)$ as defined in \labelcref{eq:softmax} and hence ``guesses" labels with higher frequent success. Such a policy can greedily and quickly decrease the cross-entropy loss but sacrifices the performance of the set-valued prediction. In contrast, the uniform policy has a higher capability of exploration, possibly leading to the fast empirical convergence of coverage rate and smaller prediction sets, even though it has an inferior capability to reduce the cross-entropy loss in each iteration.

\end{color}

\clearpage

\section{Proofs}
\begin{proof}[Proof of \cref{thm:cvg}]\label{pf:cvg}

Define $M_{t,k}:=[\Delta_{t,k}-\mathbbm{1}\{Y_t=k\}]\cdot [\alpha-\mathbbm{1}\{s^{t-1}(\bm X_t, k)< \tau^{t-1}_{k}\}]$ and 
$$\begin{aligned}
    \mathbb{V}[M_{t,k}\mid \mathcal{F}_{t-1}]&:=\mathbb{E}_{(\bm X_t, Y_t)}\left[\mathbb{E}\left[\frac{\mathbbm{1}\{A_t=k\}\cdot\mathbbm{1}\{A_t=Y_t\}}{\pi_t^2(k\mid \bm X_t)}[\alpha-\mathbbm{1}\{s^{t-1}(\bm X_t, k)< \tau^{t-1}_{k}\}]^2\mid \mathcal{F}_{t-1}, (\bm X_t, Y_t)\right]\right]\\
    &=\mathbb{E}_{(\bm X_t, Y_t)}\biggl[\frac{\mathbbm{1}\{Y_t=k\}}{\pi_t(k\mid \bm X_t)}[\alpha-\mathbbm{1}\{s^{t-1}(\bm X_t, k)< \tau^{t-1}_{k}\}]^2\mid \mathcal{F}_{t-1}\biggr]\\
    &\leq \mathbb{E}_{(\bm X_t, Y_t)}\biggl[\frac{\mathbbm{1}\{Y_t=k\}}{\pi_t(k\mid \bm X_t)}\mid \mathcal{F}_{t-1}\biggr]=\mathbb{E}\biggl[\frac{\mathbbm{1}\{Y_t=k\}}{\pi_t(k\mid \bm X_t)}\mid \mathcal{F}_{t-1}\biggr].        
\end{aligned}
$$ 

Additionally, we have  
\begin{align}
    |M_{t,k}|\leq\frac{1}{c_k}~~\mbox{and}~~\mathbb{E}[M_{t,k}\mid \mathcal{F}_{t-1}]=\mathbb{E}_{(\bm X_t, Y_t)}[\mathbb{E}[M_{t,k}\mid \mathcal{F}_{t-1}, (\bm X_t, Y_t)]]=0.
\end{align}

Then, by utilizing the Chernoff bound, for any $\xi> 0$, we have 
\begin{align}
   \mathbb{P}\biggl[\sum_{t=1}^TM_{t,k}\geq\varepsilon\biggr]&\leq\exp(-\xi\varepsilon)\cdot \mathbb{E}\biggl[\exp(\xi \sum_{t=1}^TM_{t,k})\biggr]\notag \\
   &= \exp(-\xi\varepsilon)\cdot \mathbb{E}\biggl[\mathbb{E}\biggl[\exp(\xi \sum_{t=1}^{T-1}M_{t,k} +\xi M_{T,k})\mid \mathcal{F}_{T-1}\biggr]\biggr]\notag\\
   &=\exp(-\xi\varepsilon)\cdot \mathbb{E}\biggl[\exp(\xi \sum_{t=1}^{T-1}M_{t,k})\cdot \mathbb{E}\left[\exp(\xi M_{T,k})\mid \mathcal{F}_{T-1}\right]\biggr]\notag\\
   &\leq \exp(-\xi\varepsilon)\cdot \mathbb{E}\biggl[\exp(\xi \sum_{t=1}^{T-1}M_{t,k})\cdot \exp\left(\mathbb{V}\left[ M_{T,k}\mid \mathcal{F}_{T-1}\right]c^2_k(\exp(\xi/c_k)-c^2_k-c_k\xi)\right)\biggr]\label{eq:taylorIn}\\
   &\leq \exp(-\xi\varepsilon)\cdot \exp\left(b^T_k\cdot c^2_k(\exp(\xi/c_k)-c^2_k-c_k\xi)\right)\cdot\mathbb{E}\biggl[\exp(\xi \sum_{t=1}^{T-1}M_{t,k})\biggr]\notag\\
   &\leq\exp\left( c^2_k(\exp(\xi/c_k)-c^2_k-c_k\xi)\sum_{t=1}^Tb^t_k-\xi\varepsilon\right)\notag\\
   &=\exp\left(-c^2_k\sum_{t=1}^Tb^t_k\cdot\left[-\frac{\varepsilon/c_k}{\sum_{t=1}^Tb^t_k}+(\frac{\varepsilon/c_k}{\sum_{t=1}^Tb^t_k}+1)\cdot\log(\frac{\varepsilon/c_k}{\sum_{t=1}^Tb^t_k}+1)\right]\right)\label{eq:optimalxi},
\end{align}
where \labelcref{eq:taylorIn} holds due to
\begin{align}
    \mathbb{E}\left[\exp(\xi M_{t,k})\mid \mathcal{F}_{t-1}\right]&=1+\mathbb{E}\left[\xi M_{t,k}\mid \mathcal{F}_{t-1}\right]+\mathbb{E}\left[\sum\limits_{n=2}^\infty\frac{\xi^n M^n_{t,k}}{n!}\mid \mathcal{F}_{t-1}\right]\notag\\
    &\leq1+\mathbb{E}\left[M^2_{t,k}\sum\limits_{n=2}^\infty\frac{\xi^n |M_{t,k}|^{n-2}}{n!}\mid \mathcal{F}_{t-1}\right]\notag\\
    &\leq1+\mathbb{V}[M_{t,k}\mid \mathcal{F}_{t-1}]\sum\limits_{n=2}^\infty\frac{\xi^n}{c_k^{n-2}n!}\notag\\
    &=1+\mathbb{V}[M_{t,k}\mid \mathcal{F}_{t-1}]c_k^2(\exp(\xi/c_k)-c_k^2-c_k\xi)\notag\\
    &\leq \exp\left(\mathbb{V}[M_{t,k}\mid \mathcal{F}_{t-1}]c_k^2(\exp(\xi/c_k)-c_k^2-c_k\xi)\right),\notag
\end{align}
and \labelcref{eq:optimalxi} holds since we set $\xi=c_k\log(\frac{\varepsilon/c_k}{\sum_{t=1}^Tb^t_k}+1)$.

By applying the fact of $(1+u)\log(1+u)-u\geq\frac{u^2}{2+2u/3}, ~u\geq0$ on \labelcref{eq:optimalxi}, we have $ \mathbb{P}\biggl[\sum_{t=1}^TM_{t,k}\geq\varepsilon\biggr]\leq\exp(-\frac{\varepsilon^2}{2\sum_{t=1}^Tb^t_k+2\varepsilon/(3c_k)})$, and hence $$\mathbb{P}\biggl[\left|\sum_{t=1}^TM_{t,k}\right|\geq\varepsilon\biggr]\leq2\exp(-\frac{\varepsilon^2}{2\sum_{t=1}^Tb^t_k+2\varepsilon/(3c_k)}).$$
Thus, with the probability at least $1-\delta$, we have
\begin{align}\label{eq:cvgdiff}
 &~~~~~\biggl|\sum_{t=1}^T\Delta_{t,k}\cdot [\alpha-\mathbbm{1}\{s^{t-1}(\bm X_t, k)< \tau^{t-1}_{k}\}]-\mathbbm{1}\{Y_t=k\}\cdot[\alpha- \mathbbm{1}\{s^{t-1}(\bm X_t, k)< \tau^{t-1}_{k}\}]\biggr|\notag\\
 &=\left|\sum_{t=1}^TM_{t,k}\right|\notag\\
 &\leq\frac{1}{3c_k}\log\frac{2}{\delta}+\sqrt{(\frac{1}{3c_k}\log\frac{2}{\delta})^2+2\sum_{t=1}^Tb^t_k\log\frac{2}{\delta}}\notag\\
 &\leq\frac{2}{3c_k}\log\frac{2}{\delta}+\sqrt{2\log\frac{2}{\delta}\sum_{t=1}^Tb^t_k}:=\zeta_k(T, \delta)   
\end{align}

Deriving from the updating rule for the quantile estimation in \cref{alg}, we have
\begin{align}\label{eq:itersum}
 &~~\tau^{T}_{k}=\tau^0_{k} + \eta_2  \sum\limits_{t=1}^T\Delta_{t,k}\cdot\big[\alpha- \mathbbm{1}\{s^{t-1}(\bm X_t, k)< \tau^{t-1}_{k}\}\bigr]\notag\\
    \Longrightarrow&~~\sum\limits_{t=1}^T\Delta_{t,k}\cdot [\alpha-\mathbbm{1}\{s^{t-1}(\bm X_t, k)< \tau^{t-1}_{k}\}] =\frac{\tau^{T}_{k}}{\eta_2}.
   \end{align}
Therefore, combing \eqref{eq:cvgdiff} with \eqref{eq:itersum}, with probability at least $1-\delta$, we have
\begin{equation}
    \begin{aligned}
        \frac{\tau^{T}_{k}}{\eta_2}-\zeta_k(T,\delta)&\leq\sum_{t=1}^T\mathbbm{1}\{Y_t=k\}\cdot[\alpha- \mathbbm{1}\{s^{t-1}(\bm X_t, k)< \tau^{t-1}_{k}\}]\leq \frac{\tau^{T}_{k}}{\eta_2}+\zeta_k(T,\delta)\notag\\
        \Rightarrow\frac{\tau^{T}_{k}}{\eta_2 T_k}-\frac{\zeta_k(T,\delta)}{T_k}&\leq\alpha-\sum_{t=1}^T\frac{\mathbbm{1}\{Y_t=k\}}{T_k}\cdot\mathbbm{1}\{Y_t\not\in\widehat{\mathcal{C}}^{t-1}(\bm X_t)\}\leq \frac{\tau^{T}_{k}}{\eta_2 T_k}+\frac{\zeta_k(T,\delta)}{T_k}
    \end{aligned}
\end{equation}
\end{proof}

\begin{proof}[Proof of \cref{thm:checkReg}]
Recall the definition $\tau_k^*=\min_\tau \frac{1}{T}\sum\limits_{t=1}^T\mathbbm{1}\{Y_t=k\}\rho_{\alpha}(s^{t-1}(\bm X_t), \tau)$. Thus,
\begin{align}
    T\cdot\text{Reg}_{k, \rho_\alpha}(T)&=\sum\limits_{t=1}^T\Delta_{t,k}\rho_{\alpha}(s^{t-1}(\bm X_t), \tau^{t-1}_{k})-\sum\limits_{t=1}^T\mathbbm{1}\{Y_t=k\}\rho_{\alpha}(s^{t-1}(\bm X_t), \tau_k^*)\notag\\
    &=\underbrace{\sum\limits_{t=1}^T\Delta_{t,k}\rho_{\alpha}(s^{t-1}(\bm X_t), \tau^{t-1}_{k})-\sum\limits_{t=1}^T\Delta_{t,k}\rho_{\alpha}(s^{t-1}(\bm X_t), \tau_k^*)}_{\text{Diff}_1}\notag\\
    &~~~+\underbrace{\sum\limits_{t=1}^T\Delta_{t,k}\rho_{\alpha}(s^{t-1}(\bm X_t), \tau_k^*)-\sum\limits_{t=1}^T\mathbbm{1}\{Y_t=k\}\rho_{\alpha}(s^{t-1}(\bm X_t), \tau_k^*)}_{\text{Diff}_2}\notag,
\end{align}
where $\mathbb{E}[\text{Diff}_2]=0$ since $\Delta_{t,k}$ is an unbiased estimator of $\mathbbm{1}\{Y_t=k\}$ conditional on $\mathcal{F}_{t-1}\cup(\bm X_t, Y_t)$. Additionally, we have
$$\begin{aligned}
    \text{Diff}_1&\leq\sum\limits_{t=1}^T\Delta_{t,k}\cdot g_{t-1,k}\cdot(\tau^{t-1}_k-\tau^*_k),~~~~~~~~~\text{here (sub)gradient}~ g_{t-1, k}:=-\Delta_{t,k}[\alpha-\mathbbm{1}\{s^{t-1}(\bm X_t)<\tau^{t-1}_k\}]\\
    &=\sum\limits_{t=1}^T\frac{\Delta_{t,k}}{\eta_2}\cdot (\tau^{t-1}_k-\tau^{t}_k)\cdot(\tau^{t-1}_k-\tau^*_k)\\
    &=\sum\limits_{t=1}^T\frac{\Delta_{t,k}}{2\eta_2}\cdot[(\tau^{t-1}_k-\tau^{t}_k)^2+(\tau^{t-1}_k-\tau^*_k)^2-(\tau^{t}_k-\tau^{*}_k)] \\
    &=\sum\limits_{t=1}^T\frac{\Delta_{t,k}^3\eta_2}{2}\cdot[\alpha-\mathbbm{1}\{s^{t-1}(\bm X_t)<\tau^{t-1}_k\}]^2+\sum\limits_{t=1}^T\frac{\Delta_{t,k}}{2\eta_2}\cdot[(\tau^{t-1}_k-\tau^*_k)^2-(\tau^{t}_k-\tau^{*}_k)^2], \end{aligned}$$ which further implies 
    
$$\begin{aligned}
\mathbb{E}[ \text{Diff}_1]&\leq \sum\limits_{t=1}^T\frac{\eta_2}{2}\mathbb{E}\biggl[\frac{\mathbbm{1}\{Y_t=k\}}{\pi^2_t(k\mid \bm X_t)}\biggr]+\sum\limits_{t=1}^T\frac{p_k}{2\eta_2}\cdot\mathbb{E}[(\tau^{t-1}_k-\tau^*_k)^2-(\tau^{t}_k-\tau^{*}_k)^2]\\
&=\frac{\eta_2}{2}\sum\limits_{t=1}^T\mathbb{E}\biggl[\frac{\mathbbm{1}\{Y_t=k\}}{\pi^2_t(k\mid \bm X_t)}\biggr]+\frac{p_k}{2\eta_2}\mathbb{E}[(\tau^0_k-\tau^*_k)^2-(\tau^{T}_k-\tau^{*}_k)^2]\\
&\leq \frac{\eta_2}{2}\sum\limits_{t=1}^T\mathbb{E}\biggl[\frac{\mathbbm{1}\{Y_t=k\}}{\pi^2_t(k\mid \bm X_t)}\biggr]+\frac{p_k(\tau_k^*)^2}{2\eta_2}\\
&=\tau^*_k\sqrt{p_k\sum\limits_{t=1}^T\mathbb{E}\biggl[\frac{\mathbbm{1}\{Y_t=k\}}{\pi^2_t(k\mid \bm X_t)}\biggr]}~~~\text{by choosing}~\eta_2=\tau^*_k\sqrt{\frac{p_k}{\sum\limits_{t=1}^T\mathbb{E}\bigl[\frac{\mathbbm{1}\{Y_t=k\}}{\pi^2_t(k\mid \bm X_t)}\bigr]}}
\end{aligned}$$

\end{proof}

To prove \cref{thm:checkReg_expert}, we follow a similar argument in \citet{cesa2006prediction} with two introduced lemmas. Additionally, our proof relies on the assumption that the check loss function $\rho_{\alpha}$ is bounded. It holds once the score function is bounded, e.g., the softmax, APS, and RAPS scores utilized in our study. Therefore, without loss of generality, we assume  $|\rho_{\alpha}(\cdot, \cdot)|\leq 1$.
\begin{lemma}\label{thm:lma1}
    Let $X$ be a random variable with $a\leq X\leq b$. Then for any $s\in\mathbb{R}$,
    $$\ln\mathbb{E}[\exp(sX)]\leq s\mathbb{E}[X]+\frac{s^2(b-a)^2}{8}.$$
\end{lemma}

\begin{lemma}\label{thm:lma2}
    For all $J\geq2$, for all $\beta_2\geq\beta_1\geq0$, and for all $d_j\geq0, j\in[J]$ such that $\sum_{j\in[J]}\exp(-\beta_1d_j)\geq1$, 
    $$\ln\frac{\sum_{j\in[J]}\exp(-\beta_1d_j)}{\sum_{j\in[J]}\exp(-\beta_2d_j)}\leq\frac{\beta_2-\beta_1}{\beta_1}\ln J.$$
\end{lemma}

\begin{proof}[Proof to \cref{thm:checkReg_expert}]
    For the notation simplicity, let $L^{t}_{j, k}= \sum_{t'=1}^t\Delta_{t',k}\rho_\alpha(s^{t'-1}(\bm X_{t'}, k), \tau^{t'-1})$ be the accumulative weighted check loss (up to time $t$) with $j$-th expert for class $k$, and $j^{t}_k\in\argmin_{j\in[J]}L^t_{j, k}$ denote an expert with the smallest accumulative loss up to time $t$ for class $k$. 
    After defining the weights
    $$\omega^t_{j,k}=\exp(-\frac{1}{\sqrt{t+1}}L^t_{j,k}),~~ {\omega^\prime}^t_{j,k}=\exp(-\frac{1}{\sqrt{t}}L^t_{j,k}),~~\mbox{and}~~ \bar{\omega}^t_{j,k}=\omega^t_{j,k}/\sum_{i\in[J]}\omega^t_{i,k},
    $$ we have the below equation
\begin{align}
    \sqrt{t}\ln{\bar{\omega}^{t-1}_{j^{t-1}_k,k}}-\sqrt{t+1}\ln{\bar{\omega}^{t}_{j^{t}_k,k}}&=\underbrace{(\sqrt{t+1}-\sqrt{t})\ln\frac{1}{\bar{\omega}^{t}_{j^{t}_k,k}}}_{\circled{\tiny{1}}}+
    \underbrace{\sqrt{t}\ln\frac{\bar{\omega}^{\prime ^ t}_{j^{t}_k,k}}{\bar{\omega}^{t}_{j^{t}_k,k}}}_{\circled{\tiny{2}}}+
    \underbrace{\sqrt{t}\ln\frac{\bar{\omega}^{t-1}_{j^{t-1}_k,k}}{\bar{\omega}^{\prime ^ t}_{j^{t}_k,k}}}_{\circled{\tiny{3}}},\label{eq:diff}
\end{align}
where 
\begin{align}
    \circled{\tiny{1}}\leq (\sqrt{t+1}-\sqrt{t})\ln J \label{eq:term1}
\end{align} 
since $j^t_k\in\argmin_{j\in[J]}L^t_{j, k}$ and hence $\bar{\omega}^{ t}_{j^t_k,k}\geq\frac{1}{J}$,

\begin{align}
     \circled{\tiny{2}}=\sqrt{t}\ln\frac{\sum_{j\in[J]}\exp[-\frac{1}{\sqrt{t+1}}(L^t_{j,k}-L^t_{j^{t}_k,k})]}{\sum_{j\in[J]}\exp[-\frac{1}{\sqrt{t}}(L^t_{j,k}-L^t_{j^{t}_k,k})]}&\leq \sqrt{t}\frac{\frac{1}{\sqrt{t}}-\frac{1}{\sqrt{t+1}}}{\frac{1}{\sqrt{t+1}}}\ln J~~~~\mbox{(due to \cref{thm:lma2})}\notag\\
     &=(\sqrt{t+1}-\sqrt{t})\ln J,\label{eq:term2}
\end{align}
and 
\begin{align}
     \circled{\tiny{3}}&=\sqrt{t}\ln\frac{\omega^{t-1}_{j^{t-1}_k,k}}{\omega^{\prime ^ t}_{j^{t}_k,k}}
    +\sqrt{t}\ln\frac{\sum_{j\in[J]}\omega^{\prime ^ t}_{j,k}}{\sum_{j\in[J]}\omega^{t-1}_{j,k}}\notag\\
    &=\sqrt{t}\ln\frac{\exp(-\frac{1}{\sqrt{t}}L^{t-1}_{j^{t-1}_k,k})}{\exp(-\frac{1}{\sqrt{t}}L^t_{j^{t}_k,k})}+\sqrt{t}\ln\frac{\sum_{j\in[J]}\omega^{\prime ^ t}_{j,k}}{\sum_{j\in[J]}\omega^{t-1}_{j,k}}=L^t_{j^{t}_k,k}-L^{t-1}_{j^{t-1}_k,k}+\underbrace{\sqrt{t}\ln\frac{\sum_{j\in[J]}\omega^{\prime ^ t}_{j,k}}{\sum_{j\in[J]}\omega^{t-1}_{j,k}}}_{\circled{\tiny{4}}}.\label{eq:term3}
\end{align}
Additionally, 
\begin{align}
    \circled{\tiny{4}}&=\sqrt{t}\ln\frac{\sum_{j\in[J]}\exp[-\frac{1}{\sqrt{t}}(L^{t-1}_{j,k}+\Delta_{t,k}\rho_{\alpha}(s^{t-1}(\bm X_t, k), \tau^{t-1}_{j,k}))]}{\sum_{j\in[J]}\exp(-\frac{1}{\sqrt{t}}L^{t-1}_{j,k})}\notag\\
    &=\sqrt{t}\ln\frac{\sum_{j\in[J]}\omega^{t-1}_{j,k}\cdot \exp(-\frac{1}{\sqrt{t}}\Delta_{t,k}\rho_{\alpha}(s^{t-1}(\bm X_t, k), \tau^{t-1}_{j,k}))}{\sum_{j\in[J]}\omega^{t-1}_{j,k} }\notag\\
    &=\sqrt{t}\ln\sum_{j\in[J]}\bar{\omega}^{t-1}_{j,k}\cdot \exp(-\frac{1}{\sqrt{t}}\Delta_{t,k}\rho_{\alpha}(s^{t-1}(\bm X_t, k), \tau^{t-1}_{j,k}))\notag\\
    &\leq\sqrt{t}\biggl[-\frac{1}{\sqrt{t}}\sum_{j\in[J]}\bar{\omega}^{t-1}_{j,k}\Delta_{t,k}\rho_{\alpha}(s^{t-1}(\bm X_t, k), \tau^{t-1}_{j,k}))+\frac{1}{8c_k^2t}\biggr]~~~~\mbox{(due to \cref{thm:lma1})}\notag\\
    &\leq -\Delta_{t,k}\rho_{\alpha}(s^{t-1}(\bm X_t, k), \sum_{j\in[J]}\bar{\omega}^{t-1}_{j,k} \tau^{t-1}_{j,k}))+\frac{1}{8c_k^2\sqrt{t}}\notag\\
    &=-\Delta_{t,k}\rho_{\alpha}(s^{t-1}(\bm X_t, k),  \bar\tau^{t-1}_{k}))+\frac{1}{8c_k^2\sqrt{t}}\label{eq:term4}.
\end{align}

Thus, by combing \labelcref{eq:diff,eq:term1,eq:term2,eq:term3,eq:term4}, we have 
\begin{align}
    &~~~~\Delta_{t,k}\rho_{\alpha}(s^{t-1}(\bm X_t, k), \bar\tau^{t-1}_{k})))-(L^t_{j^{t}_k,k}-L^{t-1}_{j^{t-1}_k,k})\notag\\
    &\leq \sqrt{t+1}\ln{\bar{\omega}^{t}_{j^{t}_k,k}}-\sqrt{t}\ln{\bar{\omega}^{t-1}_{j^{t-1}_k,k}}+\frac{1}{8c_k^2\sqrt{t}}+2(\sqrt{t+1}-\sqrt{t})\ln J\label{eq:simpDiff}
\end{align}
By taking the sum over $t\in[T]$ for both sides of \labelcref{eq:simpDiff}, we have
\begin{align}
    \sum_{t=1}^T \Delta_{t,k}\rho_{\alpha}(s^{t-1}(\bm X_t, k), \bar\tau^{t-1}_{k}))-L^T_{j^{T}_k, k}&\leq \ln J+\frac{1}{8c_k^2}\sum_{t=1}^T\frac{1}{\sqrt{t}}+2(\sqrt{T+1}-1)\ln J\notag\\ 
    \Longrightarrow \sum_{t=1}^T \Delta_{t,k}\rho_{\alpha}(s^{t-1}(\bm X_t, k), \bar\tau^{t-1}_{k}))-&\min_{j\in[J]}\sum_{t=1}^T\Delta_{t,k}\cdot\rho_\alpha(s(\bm X_{t}, k), \tau^{t-1}_{j,k}) \leq \frac{1}{4c_k^2}\sqrt{T}+2\sqrt{T}\ln J\notag
\end{align}
\end{proof}

\end{document}